\documentclass{article}

\usepackage[table]{xcolor}

\usepackage{microtype}
\usepackage{graphicx}
\usepackage{subfigure}
\usepackage{booktabs} 

\usepackage{tikz}
\usepackage{hyperref}

\newcommand\eg{\emph{e.g}.} 

\newcommand\ie{\emph{i.e}.}

\usepackage[accepted]{icml2025}

\usepackage{amsmath,amsfonts}
\usepackage{amssymb}
\usepackage{mathtools}
\usepackage{amsthm}

\usepackage[capitalize,noabbrev]{cleveref}

\theoremstyle{plain}

\theoremstyle{definition}

\theoremstyle{remark}

\usepackage[textsize=tiny]{todonotes}

\definecolor{myblue}{RGB}{111, 42, 221}
\definecolor{myred}{RGB}{204, 72, 158}

\usepackage{booktabs}
\usepackage{wrapfig}
\usepackage{subfigure}
\usepackage{multirow}
\usepackage{url}
\usepackage{lipsum}

\usepackage{multirow}
\usepackage{enumitem}
\definecolor{mygrey}{RGB}{224, 224, 224}
\definecolor{mycolor1}{RGB}{254, 223, 223}
\definecolor{mycolor2}{RGB}{215, 241, 217}
\definecolor{mycolor3}{RGB}{255, 251, 195}
\definecolor{loracotostarcolor}{RGB}{238, 0, 0} 
\definecolor{loracotocolor}{RGB}{255, 165, 0} 
\definecolor{lorastarcolor}{RGB}{57, 197, 187} 
\definecolor{loracolor}{RGB}{102, 204, 255} 

\usepackage{thmtools,thm-restate}
\usepackage{array}
\usepackage{etoolbox}

\AtBeginEnvironment{tabular}{\small}   
\AtBeginEnvironment{tabularx}{\small}
\AtBeginEnvironment{longtable}{\small}

\icmltitlerunning{CoTo: A Progressive Strategy to Boost Low-Rank Adaptation}

\begin{document}

\twocolumn[
\icmltitle{Come Together, But Not Right Now: \\ A Progressive Strategy to Boost Low-Rank Adaptation}

\icmlsetsymbol{equal}{*}

\begin{icmlauthorlist}
\icmlauthor{Zhan Zhuang}{sustech,cityu}
\icmlauthor{Xiequn Wang}{sustech}
\icmlauthor{Wei Li}{sustech}
\icmlauthor{Yulong Zhang}{zju}
\icmlauthor{Qiushi Huang}{sustech,surrey}
\icmlauthor{Shuhao Chen}{sustech}\\
\icmlauthor{Xuehao Wang}{sustech}
\icmlauthor{Yanbin Wei}{sustech,hkust}
\icmlauthor{Yuhe Nie}{nyu}
\icmlauthor{Kede Ma}{cityu}
\icmlauthor{Yu Zhang}{sustech}
\icmlauthor{Ying Wei}{zju}
\end{icmlauthorlist}

\icmlaffiliation{sustech}{Southern University of Science and Technology, Shenzhen, China}
\icmlaffiliation{cityu}{City University of Hong Kong, Hong Kong SAR, China}
\icmlaffiliation{zju}{Zhejiang University, Hangzhou, China}
\icmlaffiliation{hkust}{Hong Kong University of Science and Technology, Hong Kong SAR, China}
\icmlaffiliation{surrey}{University of Surrey, Surrey, UK}
\icmlaffiliation{nyu}{New York University, New York, USA}

\icmlcorrespondingauthor{Kede Ma}{kede.ma@cityu.edu.hk}
\icmlcorrespondingauthor{Yu Zhang}{yu.zhang.ust@gmail.com}
\icmlcorrespondingauthor{Ying Wei}{ying.wei@zju.edu.cn}

\icmlkeywords{Machine Learning, ICML}

\vskip 0.3in
]

\printAffiliationsAndNotice{} 

\begin{abstract}
Low-rank adaptation (LoRA) has emerged as a leading parameter-efficient fine-tuning technique for adapting large foundation models, yet it often locks adapters into suboptimal minima near their initialization. This hampers model generalization and limits downstream operators such as adapter merging and pruning. Here, we propose CoTo\footnote{This acronym nods to the Beatles' classic song `Come Together'---\textit{but not} right now.}, a progressive training strategy that gradually increases adapters' activation probability over the course of fine-tuning. By stochastically deactivating adapters, CoTo encourages more balanced optimization and broader exploration of the loss landscape. We provide a theoretical analysis showing that CoTo promotes layer-wise dropout stability and linear mode connectivity, and we adopt a cooperative-game approach to quantify each adapter's marginal contribution. Extensive experiments demonstrate that CoTo consistently boosts single-task performance, enhances multi-task merging accuracy, improves pruning robustness, and reduces training overhead, all while remaining compatible with diverse LoRA variants. Code is available at \hyperlink{https://github.com/zwebzone/coto}{https://github.com/zwebzone/coto}.
\end{abstract}
\vspace{-0.5cm}
\section{Introduction}
Parameter-efficient fine-tuning (PEFT) has become the dominant paradigm for adapting large foundation models~\citep{radford2021learning, rombach2022high,dubey2024llama} to downstream tasks. By introducing a small number of trainable parameters, such as prompts~\citep{lester2021power}, adapters~\citep{houlsby2019parameter}, or low-rank adaptation (LoRA) modules~\citep{hu2022lora}, PEFT methods achieve rapid convergence and minimal storage overhead compared to full fine-tuning. Among these, LoRA has emerged as a particularly effective approach, reparameterizing weight updates as low-rank matrices\footnote{Throughout this paper, we also call LoRA’s trainable parameters adapters. Specifically, each adapter corresponds to all LoRA parameters within a single Transformer layer.}.

\begin{figure}[t]
\centering
\includegraphics[width=\linewidth]{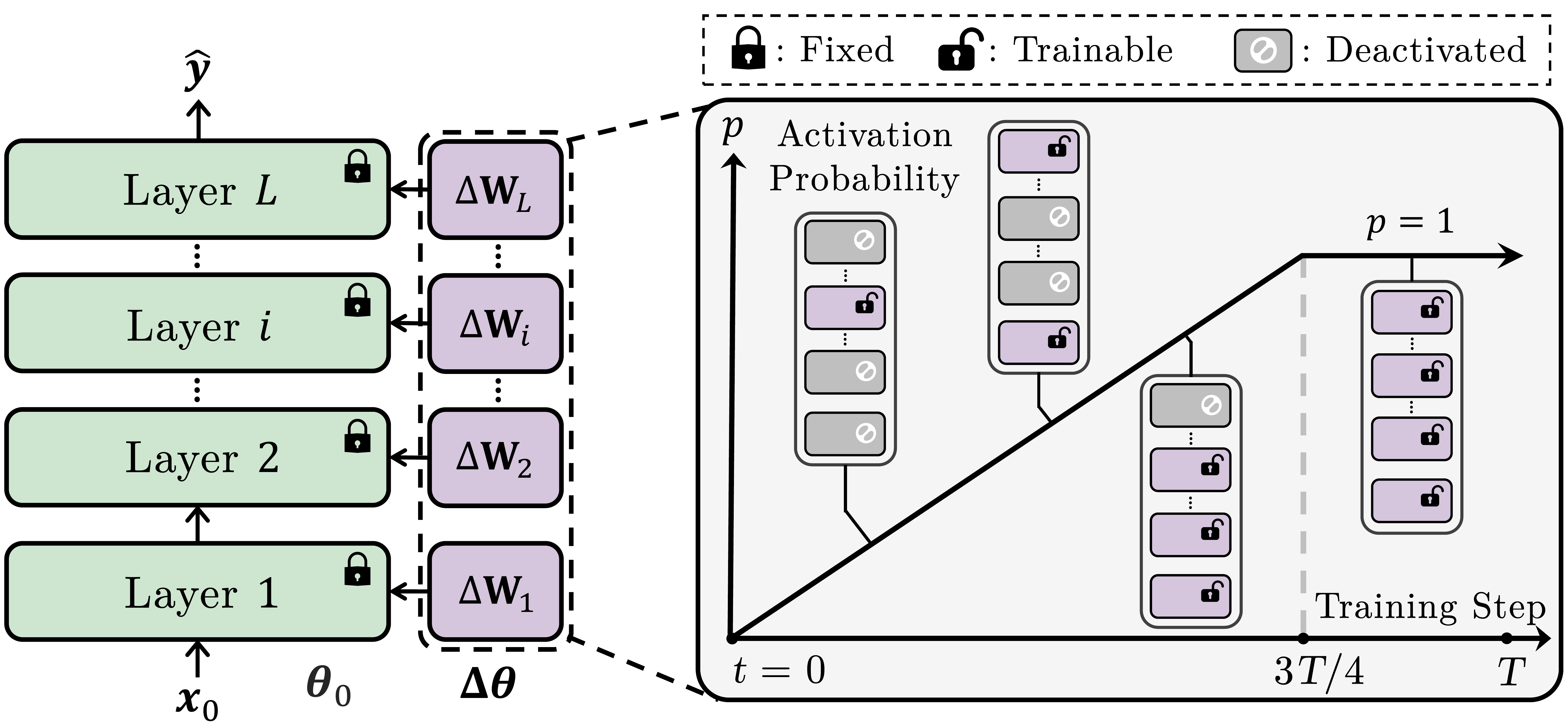}
\vskip -0.15in
\caption{
Illustration of the CoTo progressive activation schedule for LoRA adapters. For the first $75 \%$ of training (\ie, $t < 3T/4$), each adapter is stochastically deactivated (shown in gray), where the activation probability $p(t)$ increases linearly from $0$ to $1$ as training progresses. In the final $25\%$ of training, $p(t) = 1$, and all adapters remain active, reducing to full fine-tuning.}
\vskip -0.25in
\label{fig:illu2}
\end{figure}

Despite its success, vanilla LoRA often converges to suboptimal minima near their initialization, due to the ``lazy'' dynamics of standard gradient-based optimization~\citep{du2018gradient, chizat2019lazy}. Moreover, empirical studies show a pronounced layer-wise imbalance~\citep{dauphin2014identifying}: adapters in higher layers receive the bulk of the gradient signal and dominate task performance, while those in lower layers remain largely under-utilized~\citep{zhang2023adaptive, gao2024higher}. This uneven optimization not only restricts single-task generalization but also hampers downstream operations, such as adapter merging~\citep{zhao2024merging} and pruning~\citep{NIPS2015_ae0eb3ee,li2017pruning}.

To mitigate these issues, we propose CoTo, a simple progressive training strategy that gradually increases each adapter's activation probability during fine-tuning. Early in training, CoTo stochastically deactivates a random subset of adapters, forcing the model to distribute gradient updates more evenly, and then linearly raises the activation probability until all adapters participate fully. This curriculum-like scheme encourages broader exploration of the loss landscape, yields layer-wise dropout stability, and promotes linear mode connectivity (LMC) between independently trained solutions.

We provide a theoretical analysis showing that CoTo minimizes an upper bound on a weighted ensemble of subnetwork losses and, via a cooperative-game perspective, quantifies each adapter's marginal contribution using Shapley values~\citep{shapley1953value}. Empirically, CoTo consistently boosts single-task generalization, enhances multi-task adapter merging accuracy, improves adapter pruning robustness, and reduces overall training cost. Crucially, it requires no architectural changes, and integrates seamlessly with existing LoRA variants and advanced update schemes.

\section{Related Work}

\noindent{\bf Parameter-Efficient Fine-Tuning.} 
The rapid growth of foundation models has spurred extensive research into PEFT techniques, aiming to adapt large pre-trained networks to downstream tasks without incurring the computational and storage costs of full fine-tuning. Early PEFT methods introduce modular components, such as adapters~\citep{houlsby2019parameter}, prompts~\citep{lester2021power}, or prefixes~\citep{li2021prefix}, to capture task-specific knowledge while freezing the bulk of the pre-trained parameters. \citet{houlsby2019parameter} first proposed adapter layers (\ie, small bottleneck modules) inserted into Transformer blocks. Prompt-tuning~\citep{lester2021power} and prefix-tuning~\citep{li2021prefix} similarly leverage learnable tokens or continuous prefixes to steer the model toward a new task. 

LoRA~\citep{hu2022lora}, on the other hand, reframes fine-tuning as the problem of learning low-rank updates to each weight matrix. Instead of adding full-rank adapter layers, LoRA factorizes the weight update into two low-rank matrices and injects them into each Transformer layer. This decomposition dramatically reduces the number of trainable parameters while consistently delivering stronger performance than earlier PEFT methods.
Subsequent work has proposed various LoRA extensions, with the goal of improving adaptation quality, reducing parameter count further, or aligning with full fine-tuning dynamics. For example, DoRA~\citep{liudora} decomposes LoRA updates into magnitude and directional components to better approximate the full, high-dimensional updates, whereas HiRA~\citep{anonymous2024hira} applies Hadamard products between low-rank matrices and the original weights to enable high-rank adaptation without significantly increasing parameter cost. Other notable variants include LoRA-FA~\citep{zhang2023memory}, which freezes the projection-down weights for greater stability, and FourierFT~\citep{gaoparameter}, which leverages Fourier transforms to represent weight updates in the frequency domain. In parallel, adaptive rank schemes such as AdaLoRA~\citep{zhang2023adaptive}, ALoRA~\citep{liu2024alora}, and LoRA-drop~\citep{zhou2024lora}  automatically adjust the rank per layer. These variants underscore the flexibility of the low-rank paradigm but also highlight the persistent challenge of ensuring balanced, layer-wise utilization of adapters during optimization.

Beyond computational innovations, several studies have focused on improving the initialization and optimization dynamics of LoRA. PiSSA~\citep{meng2024pissa} uses truncated singular value decomposition of the pre-trained weights to initialize LoRA matrices. LoRA-GA~\citep{wang2024loraga} aligns the LoRA initialization with gradient-based approximations of full fine-tuning, while rsLoRA~\citep{kalajdzievski2023rank} adjusts scaling factors to stabilize early training. On the optimization side, LoRA+~\citep{hayoulora+} employs distinct learning rates for the two low-rank matrices. Similarly, LoRA-Pro~\citep{wang2024lorapro} modifies the gradient updates to more closely emulate the behavior of full fine-tuning. While these methods yield improvements in convergence speed or final performance, they do not explicitly address the problem of layer-wise imbalance. 
 
\noindent{\bf Model Merging.} Combining task-specific adapters to form a single set of parameters that performs well on multiple tasks relies on the property of LMC~\citep{frankle2020linear, entezarirole, zhou2023going}, which posits that two independently fine-tuned solutions often lie in loss basins connected by a low-loss linear path. In the LoRA context, LoraHub~\citep{huang2023lorahub} first demonstrates that merging low-rank adapters trained on separate language tasks can yield models with strong generalization to new tasks. Federated learning extensions, such as FedIT~\citep{zhang2024towards} and FLoRA~\citep{wang2024flora} apply LoRA merging and stacking across distributed clients, mitigating catastrophic forgetting and communication overhead. Recently, LoRA-LEGO~\citep{zhao2024merging} clusters semantically similar LoRA ``units'' within each layer before merging to reduce task interference, while ZipLoRA~\citep{shah2025ziplora} focuses on disentangling style and content subspaces to enable compositional generation in diffusion models. Despite these advances, effective multi-task merging remains challenging when adapters converge to layer-wise imbalanced minima.

\noindent{\bf Stochastic Regularization} methods, originally developed to prevent overfitting, have been adapted to the LoRA setting to encourage robustness and exploration of the parameter space. Classical techniques like Dropout~\citep{srivastava2014dropout} and DropConnect~\citep{wan2013regularization} randomly zero out elements or connections during training. Stochastic Depth~\citep{huang2016deep} and LayerDrop~\citep{Fan2020Reducing} skip entire layers with a fixed or linearly decaying probability. 
Within LoRA, entry-wise or column-wise dropout has been explored~\citep{wang2024lora, lin2024lora} to regularize low-rank matrices, but these approaches do not account for the sequential, layer-wise computation of adapters. Consequently, they may fail to correct the disproportionate updates received by higher-layer adapters. In contrast, the proposed CoTo introduces a progressive training strategy that dynamically increases the activation probability of each adapter early in training. This curriculum-like schedule balances gradient flow across all layers, fosters exploration of diverse subnetworks, and improves downstream operations.

\section{Proposed Method: CoTo}
In this section, we introduce the proposed CoTo and present two complementary perspectives to elucidate its behavior.

\subsection{Preliminaries}
Let the parameters of a pre-trained foundation model be $\boldsymbol{\theta}_0 =  \{\mathbf{W}_i\}_{i=1}^{L}$, where $\mathbf{W}_i\in \mathbb{R}^{m\times n}$ is the weight matrix of layer $i$. For an input $\boldsymbol{x}_0$, the model $f = g \circ h_L\circ\cdots\circ h_2\circ h_1$ computes a sequence of hidden features: $\boldsymbol{x}_i = h_{i}(\boldsymbol{x}_{i-1}, \mathbf{W}_{i})$ for $i\in\{1,2\ldots, L\}$, and produces the final output with the prediction head $g$: $\hat{\boldsymbol{y}} = g(\boldsymbol{x}_{L})$. In LoRA, we freeze each base weight $\mathbf{W}_{i}$ and introduce a low-rank update $\Delta \mathbf{W}_i$. Concretely, LoRA factorizes each update as $\Delta \mathbf{W}_i = \alpha \mathbf{B} \mathbf{A}$, where   $\mathbf{A} \in \mathbb{R}^{r \times n}$, $\mathbf{B} \in \mathbb{R}^{m \times r}$, and $r \ll \min(m, n)$ controls the rank.  The scaling factor $\alpha$ adjusts the magnitude of the update.

\subsection{Training Strategy} CoTo, as illustrated in Figure~\ref{fig:illu2}, introduces a simple, progressive schedule for stochastically deactivating adapters during the early stages of fine-tuning, and then gradually ``turning them on'' so that, by the final stages, all adapters participate fully. Specifically, for each layer $i$, we draw an activation indicator:
\begin{equation}
    \delta_i \sim \text{Bernoulli}(p(t)),
\end{equation}
where $p(t) \in[0,1]$ is a time-dependent probability that increases linearly from $0$ to $1$ over the first $75\%$ of training steps, and remains equal to $1$ for the remaining $25\%$. Denoting the total number of training steps by $T$, at step $t\in\{1,\ldots, T\}$, we set
\begin{equation}
    p(t) \;=\;
\begin{cases}
\dfrac{4t}{3T} & t < \dfrac{3T}{4}\\[1ex]
1 & t \ge \dfrac{3T}{4}.
\end{cases}
\end{equation}
Accordingly, the model output is adjusted to
\begin{equation}\label{coto-formula}
    \hat{\boldsymbol{y}} = f\left(\boldsymbol{x}_0; \left\{\mathbf{W}_i + \delta_i\mathbf{1} \odot\Delta \mathbf{W}_i\right\}_{i=1}^{L}\right),
\end{equation}
where $\mathbf{1}$ is an all-ones matrix of the same size as $\mathbf{W}_i$, and $\odot$ denotes the Hadamard product. The training objective is to minimize the expected loss:
\begin{equation}\label{eq:coto-objective}
    \min_{\{\Delta \mathbf{W}_i\}} \mathbb{E}_{\boldsymbol{\delta}}\left[\ell(\hat{\boldsymbol{y}}, \boldsymbol{y})\right],
\end{equation}
where $\boldsymbol{\delta} = [\delta_1,\ldots,\delta_L]^\intercal\in\{0,1\}^L$, $\boldsymbol{y}$ is the target label, and $\ell$ is typically cross-entropy loss for classification or mean squared error for regression.

\subsection{Training Dynamics}
\noindent\textbf{Curriculum of Subnetworks.} Early in training, when the probability $p$ of activating adapters is low, only a few adapters participate, forcing gradient updates to spread across layers and preventing higher‐layer adapters from dominating the loss signal; as $p$ increases, more adapters are gradually engaged, expanding the space in which the model can fine‐tune. This stochastic deactivation also counters the ``lazy'' regime---where gradients tend to stay near initialization---by encouraging exploration of a broader parameter region. By the time all adapters are active, the model has already diversified its search and is less likely to converge to poor, layer-imbalanced minima.

\noindent\textbf{Computational Savings.} Whenever $\delta_i = 0$, adapter $i$ is skipped entirely---no matrix multiplications involving $\mathbf{A}_i$ or $\mathbf{B}_i$ are performed. Thus, CoTo reduces both forward and backward computation in the early training stages. 

\begin{figure}[t]
    \centering
    \vskip -0.04in
    \includegraphics[width=0.75\linewidth]{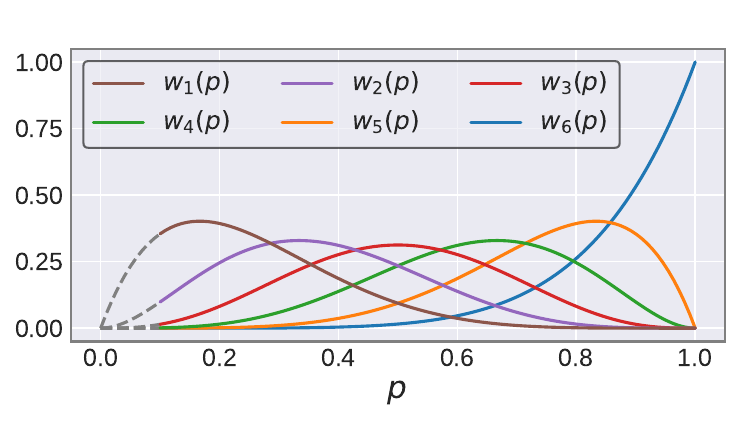}
    \vskip -0.05in
    \caption{Visualization of the weight function $w_j(p)$ in Eq.~\eqref{eq:efp}.
    }
    \label{fig:weight_function}
\end{figure}

\subsection{Progressive Optimization Perspective}
In this subsection, we view CoTo as training a weighted ensemble of partial LoRA configurations (\ie, subnetworks that omit certain adapters). This perspective makes precise how CoTo encourages both robustness to adapter dropout and improved connectivity between different minima.

Specifically, denote by 
\begin{equation}
\tilde{\boldsymbol{y}}_j = \mathbb{E}_{\|\boldsymbol{\delta}\|_1=j} \left[f\left(\boldsymbol{x}_0; \left\{\mathbf{W}_i + \delta_i \mathbf{1}\odot\Delta \mathbf{W}_i\right\}_{i=1}^{L}\right)\right],
\end{equation}
the expected model prediction over all subsets of adapters of size $j$, where $\Vert \boldsymbol{\delta}\Vert_1 = \sum_{i=1}^L \delta_i $. When $j=L$, all adapters are active and $\tilde{\boldsymbol{y}}_L$ recovers the standard LoRA model output.

\begin{table*}[t]
\vskip -0.05in
\setlength{\tabcolsep}{7pt}
\caption{Average accuracy (\%) on $11$ image classification tasks. The highest accuracy (\%) is \textbf{bolded}, while the second highest is \underline{underlined}. CLIP results are copied from~\citep{zanella2024low}. All adapters use a rank of $r=2$ with ViT-B/16 as the backbone.}
\vskip 0.1in
\label{tab:vlm_few_shot_results}
\centering
\resizebox{\textwidth}{!}{
\begin{tabular}{lcccccccccccc}
\toprule
Method & Aircraft & Caltech & Cars & DTD & EuroSAT & Flowers & Food & ImageNet & Pets & SUN & UCF & Avg \\
\midrule
CLIP 
& 24.7 & 92.9 & 65.3 & 43.6 & 47.5 & 71.4 & 86.1 & 66.7 & 89.1 & 62.6 & 66.7 & 65.1 \\ 
\midrule
LoRA {\tiny (ICLR'22)} 
& 53.89 & 96.25 & 85.12 & 72.00 & 92.05 & 97.78 & 85.15 & 73.49 & 93.27 & 76.75 & 86.72 & 82.95 \\ 
LoRA-CoTo 
& 55.69 & 96.26 & 86.04 & 72.68 & 92.97 & 98.12 & 85.48 & 73.53 & 93.42 & 76.85 & \underline{87.20} & 83.48 \\ 
\midrule
DoRA {\tiny (ICML'24)} 
& 56.21 & 96.38 & 86.67 & 72.60 & 92.13 & 98.09 & 85.04 & 73.54 & 93.59 & 76.54 & 87.15 & 83.45 \\ 
DoRA-CoTo 
& 57.35 & \underline{96.51} & 86.98 & 72.83 & \underline{93.45} & \underline{98.25} & 86.31 & 73.62 & 94.06 & 76.90 & 86.99 & 83.93 \\ 
\midrule
HiRA {\tiny (ICLR'25)} 
& \underline{57.62} & 96.35 & \underline{87.22} & \underline{73.38} & 92.51 & 98.06 & \underline{86.72} & \underline{73.76} & \underline{94.41} & \underline{76.92} & 86.81 & \underline{83.98} \\ 
HiRA-CoTo 
& \textbf{57.85} & \textbf{96.65} & \textbf{87.40} & \textbf{73.71} & \textbf{93.46} & \textbf{98.71} & \textbf{86.91} & \textbf{73.85} & \textbf{94.46} & \textbf{77.36} & \textbf{87.38} & \textbf{84.34} \\ 
\bottomrule
\end{tabular}}
\end{table*}

At iteration $t$,  CoTo samples a random vector $\boldsymbol{\delta}$, where each $\delta_i$ is $\text{Bernoulli}(p(t))$. Over this randomness, the probability that exactly $j$ adapters are active is 
\begin{equation}\label{eq:efp}
    w_j\bigl(p(t)\bigr) \!=\! \binom{L}{j}\,p(t)^j\,\bigl(1 - p(t)\bigr)^{\,L - j},  j = 0,\dots,L,
\end{equation}
as illustrated in Figure~\ref{fig:weight_function}.
\begin{restatable}{theorem}{mytheoremone}\label{mytheoremone} 
Let $\ell(\cdot, \boldsymbol{y})$ be a convex loss function. Then for any fixed $p\in[0,1]$,
\begin{equation*}
\min_{\{\Delta \mathbf{W}_i\}} \mathbb{E}_{\boldsymbol{\delta}} \left [ \ell\left(\hat{\boldsymbol{y}}, \boldsymbol{y}\right) \right ] \geq \min_{\{\Delta \mathbf{W}_i\}} \sum_{j=1}^L w_j(p) \ \ell\left(\tilde{\boldsymbol{y}}_j, \boldsymbol{y}\right).
\end{equation*}
Consequently, the expected CoTo objective at step $t$ upper-bounds a binomially weighted sum of the subnetwork losses.
\end{restatable}

Theorem~\ref{mytheoremone} follows directly from applying Jensen's inequality to the convex loss $\ell(\cdot;\boldsymbol{y})$. A detailed proof is given in Appendix~\ref{sec:appendix-theory}.

Because CoTo’s training objective accounts for all possible choices of active adapters (weighted by $w_j(p)$), the model is explicitly encouraged to perform well even if any subset of adapters is disabled. Prior work~\citep{frankle2020linear,adilova2024layerwise} shows that dropout stability often implies that independently trained solutions can be connected by a low‐loss linear path. Intuitively, because CoTo trains adapters in near‐isolation (for low $p$) before gradually re‐enabling them, each adapter’s parameters learn a solution ``locally,'' reducing inter‐adapter dependencies. Consequently, two CoTo-trained models with different random seeds tend to lie in loss valleys that are linearly connected. Empirical verification appears in Section~\ref{sec:e-lmc}.

\subsection{Cooperative-Game Perspective}\label{sec:coopereative}
An alternative way to understand CoTo is through the lens of 
 a \textit{cooperative game}: each adapter is treated as a ``player'' in a game whose ``value function'' is the model performance when that subset of adapters is active. By attributing the \textit{marginal contribution} of each adapter to overall performance, we identify precisely how CoTo encourages balanced layer‐wise optimization.

 Let $\mathcal{S} = \{1, \ldots, L\}$ index the set of $L$ adapters (one per player). For any subset $\mathcal{R}\subset \mathcal{S}$, define the value function:
 \begin{equation}\label{eq:vaf}
    \!\!\!v(\mathcal{R})
\!=\!\mathbb{E}_{\boldsymbol{x}}\!\left[\,
  \ell\left(f\left(\boldsymbol{x};\left\{\,\mathbf{W}_i + \delta_i\mathbf{1}\odot\Delta \mathbf{W}_i\right\}_{i=1}^L\right)\!,\boldsymbol{y}\right)
\right],\!
 \end{equation}
where $\delta_i = 1$ if $i\in\mathcal{R}$ and $0$ otherwise. Under this interpretation, the Shapley value~\citep{shapley1953value} of adapter $i$ can be approximated efficiently using the multilinear extension approach~\citep{owen1972multilinear}:
\begin{equation*}
\varphi_i(v)=\int_0^1 c_i(p) d p, \ c_i(p)=\mathbb{E}\left[v\left(\mathcal{R}_{i} \cup \{i\}\right)-v\left(\mathcal{R}_{i}\right)\right],
\end{equation*}
where $\mathcal{R}_i$ indexes a random subset of adapters excluding adapter $i$, and $c_i(p)$ captures the expected \textit{marginal contribution} of  adapter $i$ when selected with probability $p$. Therefore, to estimate $\varphi_i(v)$, one may sample a few values of $p$, draw random subsets $\mathcal{R}_i$, compute the difference $v\left(\mathcal{R}_{i} \cup \{i\}\right)-v\left(\mathcal{R}_{i}\right)$ (again approximated by averages over a finite set of samples in Eq.~\eqref{eq:vaf}), and average appropriately. By inspecting $\varphi_i(v)$ after CoTo training, we gain insights into each adapter's marginal contribution.

\begin{table*}[!t]
\centering
\vskip -0.05in
\caption{Average accuracy (\%) on $8$ commonsense reasoning tasks~\citep{hu2023llm}  using LLaMA-2-7B and LLaMA-3-8B backbones. All adapters use a rank of $r=32$. Results without CoTo are copied from~\citep{anonymous2024hira}.}
\setlength{\tabcolsep}{9.5pt}
\vskip 0.08in
\resizebox{\textwidth}{!}{%
\begin{tabular}{llccccccccc}
\toprule
Model & Method & ARC-c & ARC-e & BoolQ & HellaS & OBQA & PIQA & SIQA & WinoG & Avg \\ \midrule
ChatGPT & N/A &
79.90 & 89.80 & 73.10 & 78.50 & 74.80 & 85.40 & 68.50 & 66.10 & 77.01 \\ \midrule
\multirow{6}{*}{LLaMA-2-7B}
& LoRA &
64.70 & 79.80 & 69.80 & 83.60 & 81.00 & 79.90 & 79.50 & 82.60 & 77.61 \\
& LoRA-CoTo &
69.58 & 85.33 & 71.48 & \textbf{91.15} & 82.10 & 82.89 & 78.94 & 83.54 & 80.63 \\
& DoRA &
68.20 & 83.70 & 71.80 & 89.10 & 82.40 & 83.70 & 76.00 & 82.60 & 79.69 \\
& DoRA-CoTo &
69.51 & 85.08 & \textbf{72.25} & \underline{90.82} & 81.33 & 83.10 & 79.50 & 83.43 & 80.64 \\
& HiRA &
\underline{73.81} & \underline{86.74} & 71.22 & 88.12 & \textbf{84.60} & \underline{83.35} & \underline{79.53} & \underline{83.98} & \underline{81.42} \\
& HiRA-CoTo &
\textbf{74.49} & \textbf{87.08} & \underline{72.11} & 88.40 & \underline{84.00} & \textbf{84.33} & \textbf{79.89} & \textbf{85.24} & \textbf{81.94} \\ \midrule
\multirow{6}{*}{LLaMA-3-8B}
& LoRA &
71.20 & 84.20 & 70.80 & 91.70 & 79.00 & 85.20 & 79.90 & 84.30 & 80.79 \\
& LoRA-CoTo &
79.35 & 90.81 & 75.02 & 94.77 & 85.20 & 88.39 & 80.55 & 86.08 & 85.02 \\
& DoRA &
80.40 & 90.50 & 74.60 & 95.50 & 85.80 & 89.30 & 79.90 & 85.60 & 85.20 \\
& DoRA-CoTo &
79.38 & \underline{91.50} & \textbf{75.40} & \textbf{95.98} & 86.00 & 88.52 & 81.12 & 86.00 & 85.49 \\
& HiRA &
\underline{82.90} & \textbf{93.27} & \textbf{75.40} & 95.36 & \underline{88.32} & \underline{89.70} & \underline{81.15} & \underline{87.70} & \underline{86.72} \\
& HiRA-CoTo &
\textbf{83.36} & \textbf{93.27} & \underline{75.32} & \underline{95.42} & \textbf{88.40} & \textbf{90.15} & \textbf{81.99} & \textbf{88.08} & \textbf{87.00} \\
\bottomrule
\end{tabular}}
\vskip -0.05in
\label{tab:llm_commensence}
\end{table*}

\section{Experiments}
We evaluate CoTo's effectiveness through a series of experiments designed to answer three key questions: 1) Can CoTo improve single-task generalization across diverse benchmarks? 2) Does CoTo facilitate LMC for more effective model merging? 3) Can CoTo enhance pruning robustness? All experiments use three random seeds to ensure statistical reliability, and implementation details are deferred to Appendix~\ref{sec:appendix-training-details}.

\subsection{Single-Task Generalization}\label{sec:e-benchmark}

\noindent\textbf{Results on Vision Benchmarks.}
To assess CoTo's impact in the vision domain, we follow the CLIP‐LoRA setup~\citep{zanella2024low}, fine-tuning the ViT-B/16 backbone on $11$ image classification datasets assembled by~\citet{zhou2022learning}. Each dataset defines an independent few-shot task with $16$ training images per class. We compare three LoRA variants---vanilla LoRA, DoRA~\citep{liudora}, and HiRA~\citep{anonymous2024hira}---both with and without  CoTo. Table~\ref{tab:vlm_few_shot_results} reports average accuracies over three seeds. CoTo yields noticeable performance gains across all LoRA variants, demonstrating that progressive training leads to more balanced utilization of all adapters.

\begin{table}[!t]
\centering
\setlength{\tabcolsep}{9pt}
\vskip -0.1in
\caption{Average accuracy (\%) on mathematical reasoning tasks~\citep{cobbe2021training} using the LLaMA-2-7B backbone. All adapters use a rank of $r=8$. Results {without CoTo} are copied from~\citep{wang2024lorapro}.}
\vskip 0.08in
\resizebox{\linewidth}{!}{
\begin{tabular}{lccc}
\toprule
{Method} & {Params} (\%) & w/o CoTo  & w/ CoTo  \\ \midrule
LoRA  & 0.296 &42.08 \small{± 0.04} & 55.85 \small{± 0.74} \\
DoRA  & 0.316 &53.07 \small{± 0.75} & 56.56 \small{± 0.19} \\
HiRA  & 0.296 &  54.51 \small{± 0.59} & 56.68 \small{± 0.09} \\
\midrule
PiSSA & 0.296 &44.54 \small{± 0.27} & 50.16 \small{± 0.47}\\
rsLoRA & 0.296 & 45.62 \small{± 0.10}  &56.99 \small{± 0.66} \\
LoRA+ & 0.296 &  52.11 \small{± 0.62} & 54.36 \small{± 0.43}\\
LoRA-Pro & 0.296 & 54.23 \small{± 0.79} & 57.16 \small{± 0.38} \\
\bottomrule
\end{tabular}
}
\vskip -0.12in
\label{tab:llm_mathcode}
\end{table}

\noindent\textbf{Results on Language Benchmarks.} In the language domain, we first evaluate CoTo on commonsense reasoning tasks using LLaMA backbones~\citep{touvron2023llama, dubey2024llama}. Following~\citet{wang2024lorapro}, we fine-tune LLaMA-2-7B and LLaMA-3-8B on the Commonsense170K suite~\citep{hu2023llm}, which comprises $8$ tasks. All adapter variants use a rank of $r=32$. Table~\ref{tab:llm_commensence} reports average accuracies, from which we observe consistent performance improvements across different backbones, LoRA variants, and task complexities. These results indicate that CoTo's progressive activation schedule helps mitigate layer‐wise imbalance and lazy convergence, especially as model capacity increases.

We further evaluate CoTo on mathematical reasoning tasks by fine-tuning LLaMA-2-7B on MetaMathQA~\citep{yumetamath} and testing it on GSM8K~\citep{cobbe2021training}. In this setting, we adjust the adapter rank to $8$ and compare CoTo against  PiSSA~\citep{meng2024pissa}, rsLoRA~\citep{kalajdzievski2023rank}, LoRA+~\citep{hayoulora+}, and LoRA-Pro~\citep{wang2024lorapro}. Similar performance gains have been achieved, as shown in Table~\ref{tab:llm_mathcode}. 

\subsection{Single-Task and Multi-Task Model Merging}\label{sec:e-lmc}

\begin{figure}[t]
    \centering
    \includegraphics[width=0.95\linewidth]{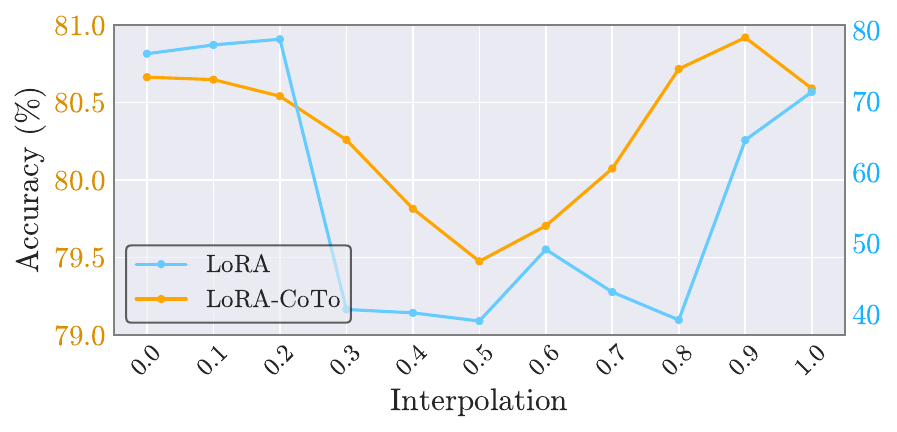}
    \vskip -0.14in
    \caption{Linear interpolation accuracy on commonsense reasoning tasks~\citep{hu2023llm}. CoTo's interpolation curve averaged across $8$ tasks (orange) remains flatter and higher compared to vanilla LoRA (blue), demonstrating superior LMC.}
    \vskip -0.15in
    \label{fig:enter-label2}
\end{figure}

\begin{figure*}[t]
\centering
\includegraphics[width=0.99\linewidth]{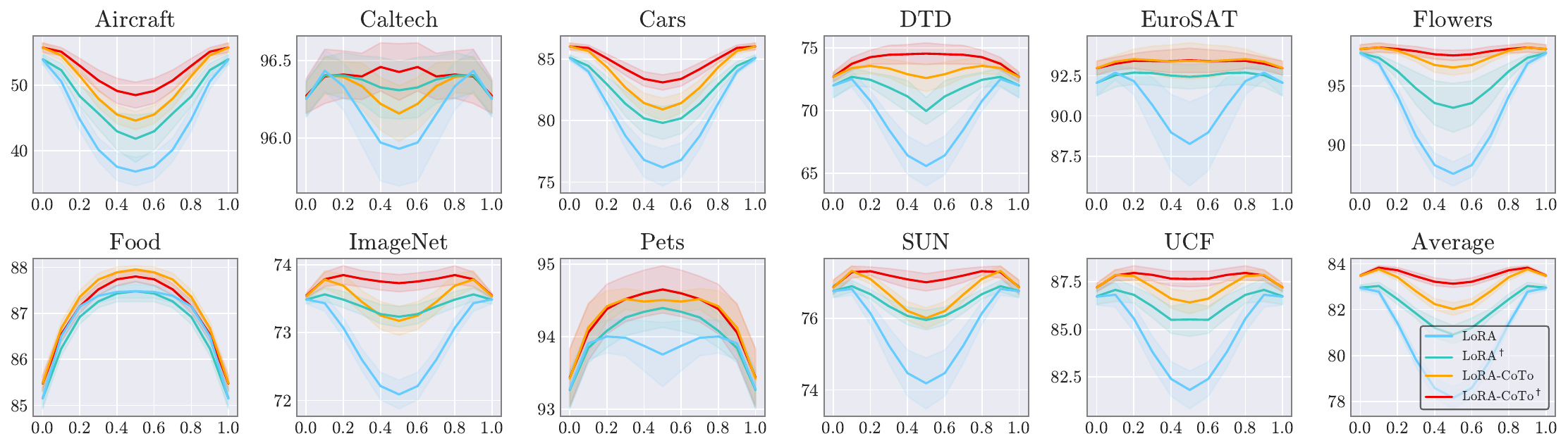}
\vskip -0.05in
\caption{Linear interpolation accuracy on $11$ image classification tasks. CoTo (orange) consistently outperforms vanilla LoRA (blue). Adding alignment (denoted by $^\dagger$) improves both but preserves CoTo's margin.}
\vskip -0.05in
\label{fig:loramerge}
\end{figure*}

In this subsection, we assess the LMC property through two complementary experiments.

\noindent\textbf{Single-Task Model Merging.} To quantify how well two independently trained LoRA solutions can be connected by a linear path, we fine-tune two instances on the same task from different random seeds, and then linearly interpolate their adapter parameters for interpolation ratios $\lambda \in [0, 1]$.  Figure~\ref{fig:enter-label2} illustrates interpolation accuracies on commonsense reasoning tasks~\citep{hu2023llm}: at $\lambda = 0.5$ (equal mixture), LoRA-CoTo maintains $79\%$ accuracy, whereas vanilla LoRA drops to $39\%$. Across the entire interpolation trajectory, CoTo’s curve remains substantially flatter and higher, indicating that independently CoTo-trained adapters lie in closer, low-loss basins. Analogous trends emerge in the image classification experiments (see Figure~\ref{fig:loramerge}). Even after applying an additional ``alignment'' step: learning an invertible matrix $\mathbf{P}$ to minimize $\lVert\Delta \mathbf{W}_f- \Delta \mathbf{W}_e\rVert_2$, where $\Delta\mathbf{W}_f$ and $\Delta\mathbf{W}_e$ denote linear weight fusion and model ensemble, respectively, CoTo retains its advantages over vanilla LoRA\footnote{Linear weight fusion is computed by $\Delta \mathbf{W}_f= (\lambda \mathbf{B}_1 + (1-\lambda) \mathbf{B}_2)(\lambda \mathbf{A}_1 + (1-\lambda) \mathbf{A}_2 )$, which preserves the rank, while linear model ensemble~\citep{zhao2024loraretriever} is computed by $\Delta \mathbf{W}_e  = \lambda \mathbf{B}_1\mathbf{A}_1 + (1-\lambda) \mathbf{B}_2\mathbf{A}_2$. To insert the learnable invertible matrix $\mathbf{P}$, we reparameterize $\Delta \mathbf{W}_2 = \mathbf{B}_{2}\mathbf{A}_{2}= (\mathbf{B}_{2} \mathbf{P}) (\mathbf{P}^{-1} \mathbf{A}_{2})$.}. Analysis of $\lVert\Delta \mathbf{W}_f - \Delta \mathbf{W}_e\rVert_2$ and  $\lVert \mathbf{P}\rVert_2$ in Figure~\ref{fig:p-analysis} further confirms that CoTo's performance gains stem from balanced layer-wise optimization rather than mere post-hoc alignment.

\begin{table*}[t]
    \centering
    \vskip -0.05in
    \caption{Average accuracy (\%) on multi-task merging for $9$ generative language understanding tasks~\citep{wang2019glue} using LLaMA-2-7B and LLaMA-2-13B backbones. ``Fusion'' and ``Ensemble'' correspond to the linear weight fusion and linear model ensemble that compute $\Delta \mathbf{W}_f$ and $\Delta \mathbf{W}_e$, respectively. Results without CoTo are copied from~\citep{zhao2024merging}.
    }
    \vskip 0.08in
    \resizebox{\linewidth}{!}{
    \begin{tabular}{llcccccccccc}
    \toprule
    \multirow{2}{*}{Model} & \multirow{2}{*}{Method}  & \multicolumn{7}{c}{In-Domain Task} & \multicolumn{2}{c}{Out-of-Domain Task} &  \multirow{2}{*}{Avg}\\
    \cmidrule(lr){3-9}
    \cmidrule(lr){10-11}
     &  & CoLA & MNLI & MRPC & QNLI & QQP & RTE & SST2 & SNLI & WNLI & \\
    \midrule
    \multirow{7}{*}{LLaMA-2-7B} &  Task-Specific LoRA & 61.63 & 77.46 & 68.00 & 77.25 & 75.83 & 52.22 & 75.74 & -- & -- & -- \\
    \cmidrule(lr){2-12}
    & Fusion & 54.42 & 36.09 & \underline{68.00} & 44.41 & 51.72 &  48.15 & 42.99 & 31.64 & 47.14 & 47.17 \\
    & Fusion-CoTo & \textbf{57.31} & 47.39 & 61.75 & \underline{62.60} & \underline{71.89} & 71.11 & 60.80 & 36.79 & \underline{57.14} & 58.53 \\
    &Ensemble & 55.67 & 45.89 & 59.25 & 59.84 & 67.38 & 68.89 & 66.44 & 36.73 & 51.43 & 56.84 \\
     &Ensemble-CoTo & \underline{57.21} & 45.68 & 47.75 & 61.39 & 68.59 & 69.26 & 60.57 & 35.24 & \textbf{64.29} & 56.66 \\
    & LoRA-LEGO & 55.48 & \underline{55.73} & 66.00 & 62.29 & 71.07 & \textbf{71.85} & \underline{73.22} & \underline{51.36} & 52.86 & \underline{62.21} \\
   & LoRA-LEGO-CoTo &  53.94 & \textbf{64.35} & \textbf{72.25} & \textbf{72.71} & \textbf{78.51} & \underline{71.48} & \textbf{75.75} & \textbf{58.59} & \underline{57.14} & \textbf{67.19} \\
    \midrule
   \multirow{7}{*}{LLaMA-2-13B} &  Task-Specific LoRA & 69.04 & 88.23 & 89.25 & 82.33 & 86.29 & 80.74 & 76.44 & -- & -- & --\\
       \cmidrule(lr){2-12}
   & Fusion & 45.48 &46.32 & 67.75 & 46.68 & 47.50 & 62.96 & 46.78 & 42.42 & 42.86 & 49.86 \\
    & Fusion-CoTo & \textbf{64.52} & 57.82 & 73.75 & 66.10 & 78.53 & 75.93 & 75.52 & 42.28 & \textbf{67.14} & 66.84
\\
    &Ensemble & 62.50 & 64.64 & 74.75& 71.81 &81.35 &\textbf{79.26} & 75.52 &54.32&60.00 &69.35 \\
    &Ensemble-CoTo & \underline{63.75} &  60.54 &  71.75 &  67.82 &  76.38 &  77.78 &  75.75 &  46.79 &  62.86 &  67.05  \\
    & LoRA-LEGO & 59.42& \underline{65.40} &\underline{75.50} &\underline{72.29} &\underline{82.51} &\underline{78.52} &\underline{75.98} &\underline{58.54}& 64.29 &\underline{70.27} \\
   & LoRA-LEGO-CoTo &  61.83 & \textbf{65.75} & \textbf{78.25} & \textbf{76.81} &  \textbf{82.90} &  77.78 &  \textbf{76.32} &  \textbf{58.74} &  \underline{65.71} &  \textbf{71.57} \\
\bottomrule
\end{tabular}
}
\vskip -0.1in
\label{tab:llmloralego}
\end{table*}

\noindent{\bf Multi-Task Model Merging.} 
Building on CoTo's enhanced LMC, we next examine its impact on merging adapters trained on different tasks following the experimental setup and default configurations in~\citep{zhao2024merging}. First, we consider multi-task merging for generative language understanding tasks~\citep{wang2019glue}. We test on seven in-domain tasks and two out-of-domain tasks~\citep{longpre2023flan} via the same prompt format~\citep{weifinetuned}. Both LLaMA-2-7B and LLaMA-2-13B backbones are used. We employ three merging strategies: linear weight fusion (\ie, $\Delta\mathbf{W}_f$), linear model ensemble (\ie, $\Delta\mathbf{W}_e$), and the LoRA-LEGO method proposed by~\citet{zhao2024merging}, which explicitly aligns and fuses parameter updates. As reported in Table~\ref{tab:llmloralego}, CoTo-trained adapters yield markedly better merging performance. On LLaMA-2-7B, linear weight fusion of CoTo-based adapters improves average accuracy from $47.17 \%$ to $58.53\%$ ($+11.36\%$), and LoRA-LEGO merging rises from $62.21 \%$ to $67.19\%$ ($+4.98\%$). The ensemble approach also benefits some tasks, though gains are more modest. Similar trends are observed for LLaMA-2-13B, indicating that CoTo encourages each adapter to converge to parameters that lie in closer, low-loss subspaces.

\begin{figure}[t]
    \centering
    \includegraphics[width=\linewidth]{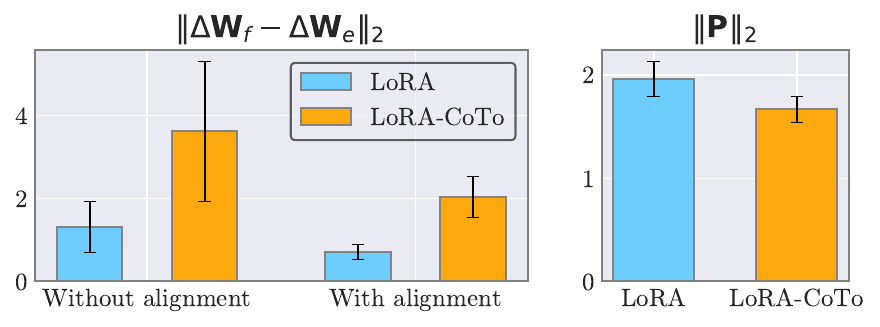}
    \vskip -0.1in
    \caption{Analysis of the optimized alignment matrix $\mathbf{P}$ for LoRA$^\dagger$ and LoRA-$\mathrm{CoTo}^\dagger$.  Error bars denote standard deviations. }
    \vskip -0.1in
    \label{fig:p-analysis}
\end{figure}

We also evaluate CoTo's efficacy for merging adapters on six discriminative language understanding tasks~\citep{wang2019glue} using DeBERTa-v3~\citep{hedebertav3}.  Although merging classifiers across distinct tasks is inherently more challenging due to differences in feature pooling and output prediction, CoTo-trained adapters still exhibit consistent improvements across all three merging strategies (see Table~\ref{tab:llmloradeberta} in the Appendix). These findings indicate that CoTo is compatible with existing merging techniques and consistently enhances multi-task LoRA merging across both generative and discriminative architectures.

\begin{figure}[t]
\centering
\includegraphics[width=0.85\linewidth]{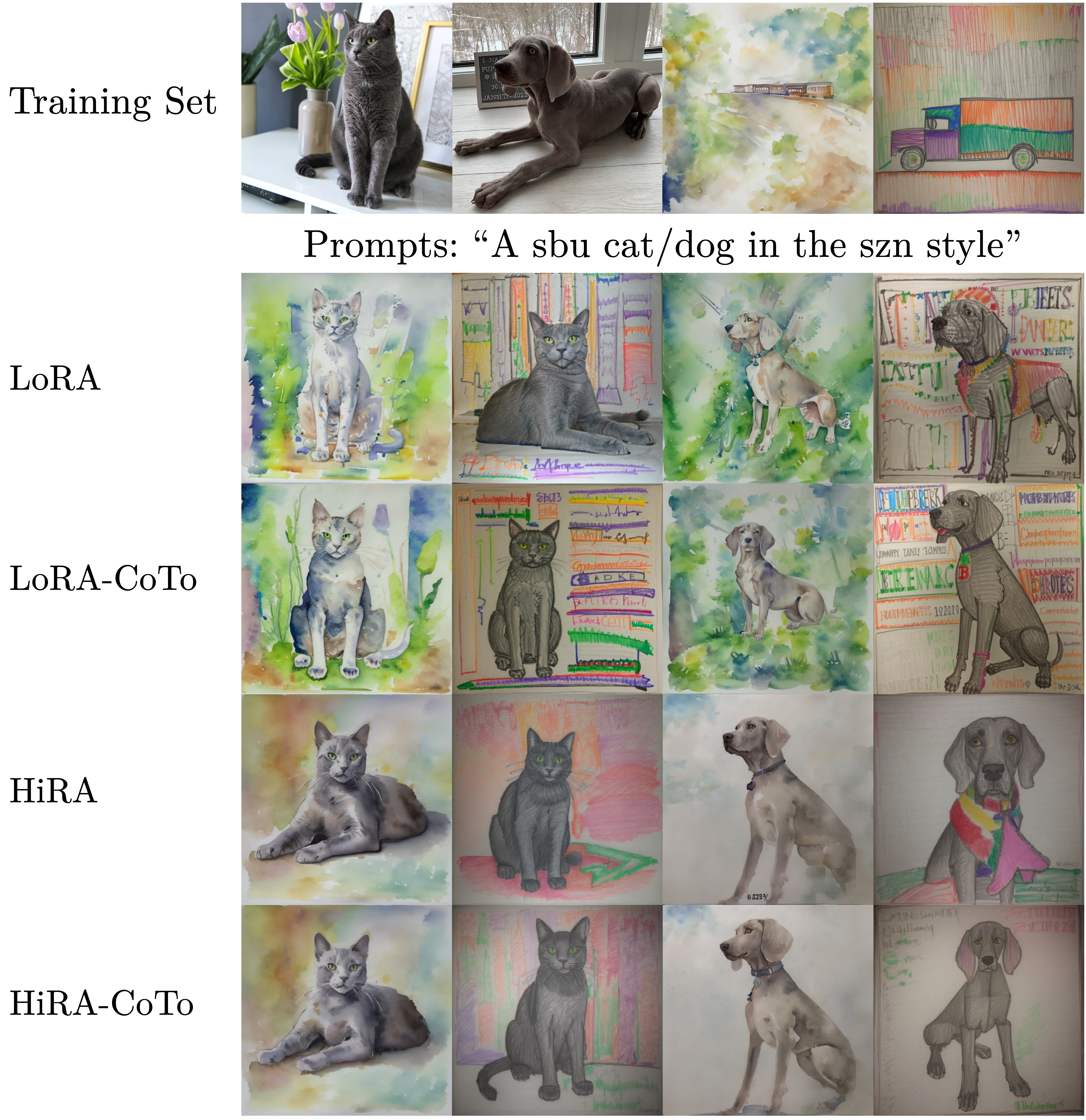}
\caption{Customized sample images generated by SDXL~\citep{podell2024sdxl} with and without CoTo. When merging style and object adapters via ZipLoRA~\citep{shah2025ziplora}, CoTo preserves both the object identity and artistic style more faithfully than vanilla LoRA. Each comparison uses the same seed.}
\vskip -0.1in
\label{fig:diffusion}
\end{figure}

Finally, we explore CoTo in the context of diffusion-based generative models. Using SDXL~\citep{podell2024sdxl} as our backbone, we apply DreamBooth~\citep{ruiz2023dreambooth} to fine-tune separate style and object adapters---specifically, object LoRAs for two categories (\ie, cat and dog) and style LoRAs for two artistic styles (\ie, watercolor and crayon). We then merge style and object adapters using ZipLoRA~\citep{shah2025ziplora}. Qualitative results in Figure~\ref{fig:diffusion} demonstrate that CoTo significantly reduces style and object forgetting: for instance, a crayon-style cat generated with LoRA-CoTo clearly preserves both the cat’s identity and ``crayon-ness,'' whereas vanilla LoRA often compromises one or the other. This qualitative evidence further attests to CoTo's ability to learn adapters that merge more coherently across diverse vision and language tasks.

\subsection{Model Pruning}\label{sec:e-prune}
The stochastic nature of CoTo naturally lends itself to improved pruning robustness, as adapters are trained to maintain performance even when a random subset is deactivated. To systematically evaluate this property, we conduct both structured and unstructured 
 pruning experiments. We first examine layer-wise structured pruning on the visual texture classification task~\citep{cimpoi2014describing} by selectively removing adapters from different network layers. As shown in the left panel of Figure~\ref{fig:loraprune}, we compare four configurations: 1) removing alternating layers (denoted by $\mathrm{EveryOther}$), 2) pruning the first $4$ layers (denoted by $\mathrm{Low}$), 3) pruning middle $4$ layers (denoted by $\mathrm{Middle}$), and 4) pruning the last $4$ layers (denoted by $\mathrm{High}$). The results demonstrate that CoTo-trained adapters maintain significantly better performance across all pruning patterns compared to vanilla LoRA. Notably, the ``Early LoRA-CoTo'' checkpoint (sampled right after the first $25\%$ of training) already shows strong pruning robustness, indicating that the benefits emerge early in the progressive training schedule. Complete results are detailed in Figure~\ref{fig:pruning-complete-dataset} of the Appendix.

\begin{figure}[!t]
  \centering
  \subfigure{\label{fig:loraprune-structure}\includegraphics[width=0.48\linewidth]{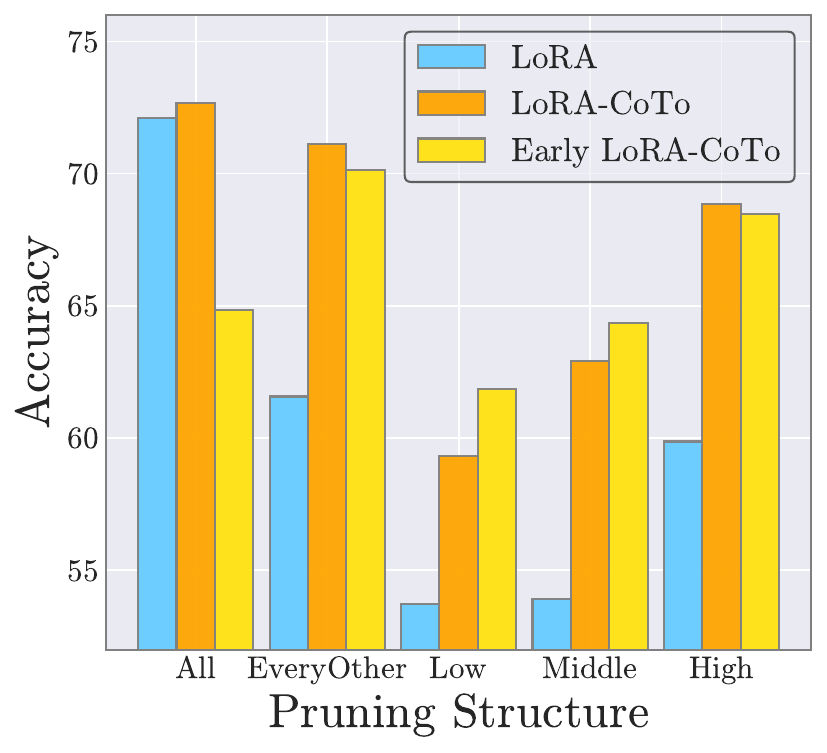}}\hfill
  \subfigure{\label{fig:loraprune-dropout}\includegraphics[width=0.48\linewidth]{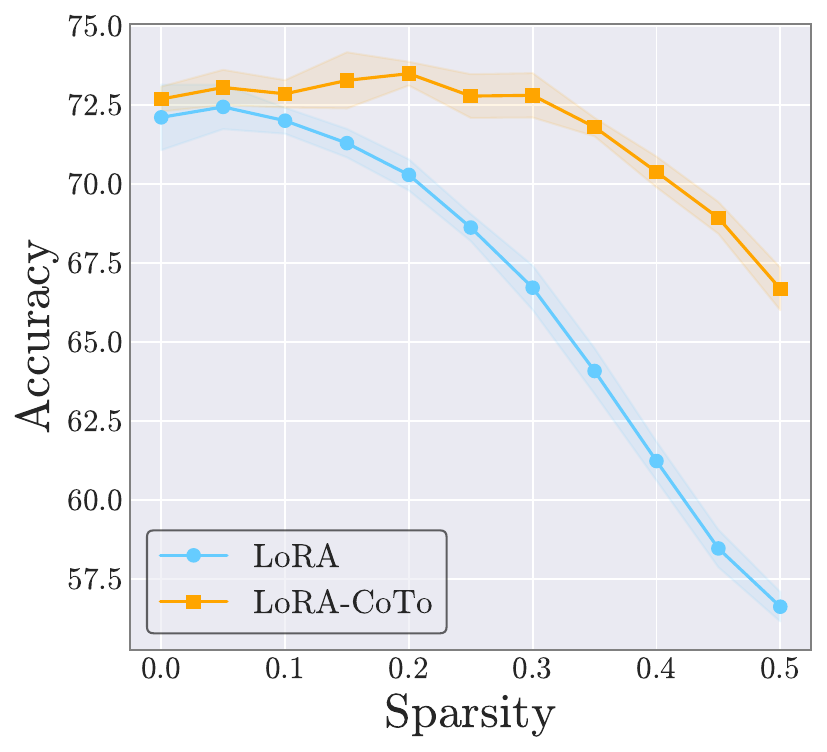}}
  \vskip -0.12in
  \caption{Average accuracy (\%) on model pruning for the visual texture classification task~\citep{cimpoi2014describing}. Left panel: Structured pruning applied to LoRA, LoRA-CoTo, and Early LoRA-CoTo under four pruning patterns: alternating layers ($\mathrm{EveryOther}$), first $4$ layers ($\mathrm{Low}$), middle $4$ layers ($\mathrm{Middle}$), and last $4$ layers ($\mathrm{High}$). ``$\mathrm{All}$'' denotes unpruned models. Right panel: Unstructured pruning with varying sparsity.}
  \vskip -0.2in
  \label{fig:loraprune}
\end{figure}

For fine-grained sparsity analysis, we evaluate unstructured pruning by zeroing out increasing percentages of adapter parameters. The right panel of Figure~\ref{fig:loraprune} shows that the performance gap widens with increasing sparsity level, and at $50\%$ sparsity, LoRA-CoTo achieves $10\%$ higher accuracy than vanilla LoRA. These findings highlight that CoTo enables more aggressive pruning while maintaining model utility even at high sparsity levels.
\begin{table}[!t]
\centering
\vskip -0.05in
\caption{Mean Euclidean distance between LoRA adapter weights across four learning rates. {Init. to Final}: Distance from initial to final weights. {Final to Final}: Distance between final weights from different initialization. Note that the mean initial distances are $1.155$ for independent random seeds and $0.02$ for perturbed seeds.}
\vskip 0.08in
\setlength{\tabcolsep}{9.5pt}
\resizebox{\linewidth}{!}{
\begin{tabular}{lcccc}
\toprule
{Distance} & 5e-5 & 1e-4  & 5e-4  & 1e-3 \\ \midrule
\rowcolor{mygrey} \multicolumn{5}{l}{Random Initialization}\\
Init. vs. Final (w/o CoTo)  & 0.48 & 0.45 & 0.76 & 1.32  \\
Init. vs. Final (w/ CoTo)  & 0.61 & 0.79  & 1.25 & 1.64 \\
Final vs. Final (w/o CoTo)  & 1.70 & 1.81 & 2.35 & 3.12  \\	
Final vs. Final (w/ CoTo)  & 1.38 & 1.53  &	2.14 & 2.63  \\
\midrule
\rowcolor{mygrey} \multicolumn{5}{l}{Same Initialization with Minor Additive Uniform Noise}\\
Final vs. Final (w/o CoTo)  & 0.05 & 	0.07 	& 0.59 	& 1.71 \\
Final vs. Final (w/ CoTo) &	0.04	& 0.06 &	0.21	& 0.52  \\
\bottomrule
\end{tabular}
}
\vskip -0.1in
\label{tab:init-comp}
\end{table}

\subsection{Further Analysis on DTD} 
\noindent\textbf{Convergence Near Initialization.} 
We empirically validate that LoRA adapters converge near their initialization, consistent with ``lazy training'' dynamics~\citep{chizat2019lazy}. Using t-SNE visualization~\citep{van2008visualizing} (see Figure~\ref{fig:tsne} in the Appendix), we find that final adapter weights form tight clusters around their respective initialization points across five random seeds and four learning rates. This local convergence persists regardless of whether CoTo is applied. Nevertheless, CoTo yields slightly larger distances (see Table~\ref{tab:init-comp}), suggesting broader exploration. Moreover, final weights from independent random seeds are closer under CoTo, indicating more consistent convergence paths. When initialized from the same point with minor additive uniform noise, CoTo-trained adapters converge to tighter clusters compared to vanilla LoRA, demonstrating robustness to initialization perturbations.

\noindent\textbf{Adapter Contribution Analysis.}
To quantify layer-wise marginal utilization, we compute approximated Shapley values~\citep{shapley1953value}  for each adapter via multilinear extension~\citep{owen1972multilinear} on the visual texture classification task~\citep{cimpoi2014describing} again for its representativeness and computational feasibility. As shown in Figure~\ref{fig:shapley}, vanilla LoRA exhibits skewed contributions, with $69\%$  concentrated in the highest $4/12$ Transformer layers. In contrast, LoRA-CoTo and Early LoRA-CoTo achieve more balanced contributions ($\pm 8\%$ and $\pm 3\%$ deviations across layers), confirming its efficacy in mitigating gradient imbalance. 

\begin{figure}[!t]
\centering
\includegraphics[width=\linewidth]{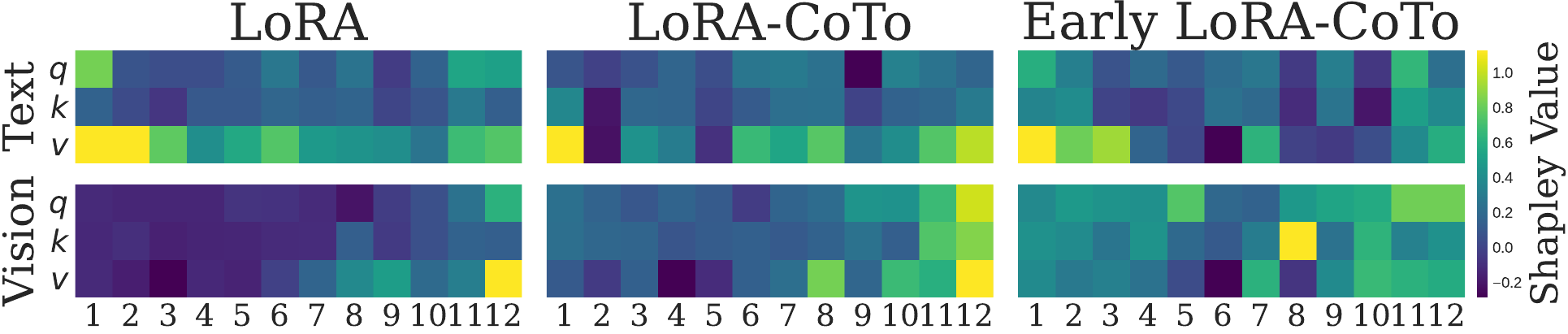}
\vskip -0.1in
\caption{Approximated Shapley values of LoRA adapters by multilinear extension~\citep{owen1972multilinear}.}
\label{fig:shapley}
\end{figure}

\subsection{Ablation Studies}
To systematically evaluate the design choices of CoTo and its robustness to hyperparameter variations, we conduct key ablation experiments on mathematical reasoning tasks~\citep{cobbe2021training}. All experiments employ LoRA-Pro~\citep{wang2024lorapro} as the baseline.

\noindent{\bf Training Phase Transition.} We first investigate the impact of varying the proportion of training time
allocated to the stochastic activation phase (see left panel of Figure~\ref{fig:mainfig}). The $x$-axis represents the percentage of total training spent in the first phase (where $p(t) < 1$), with $0\%$ corresponding to vanilla LoRA (\ie, without CoTo) and $100\%$ representing training exclusively with stochastic activation (\ie, $p(t)$ never reaches $1$). Our results demonstrate that a $75\%$ first-phase proportion strikes a good balance between task performance and training efficiency with secondary benefits like improved LMC.

\begin{figure}[t]
    \centering
    \vskip -0.05in
    \subfigure{
        \includegraphics[width=0.47\linewidth]{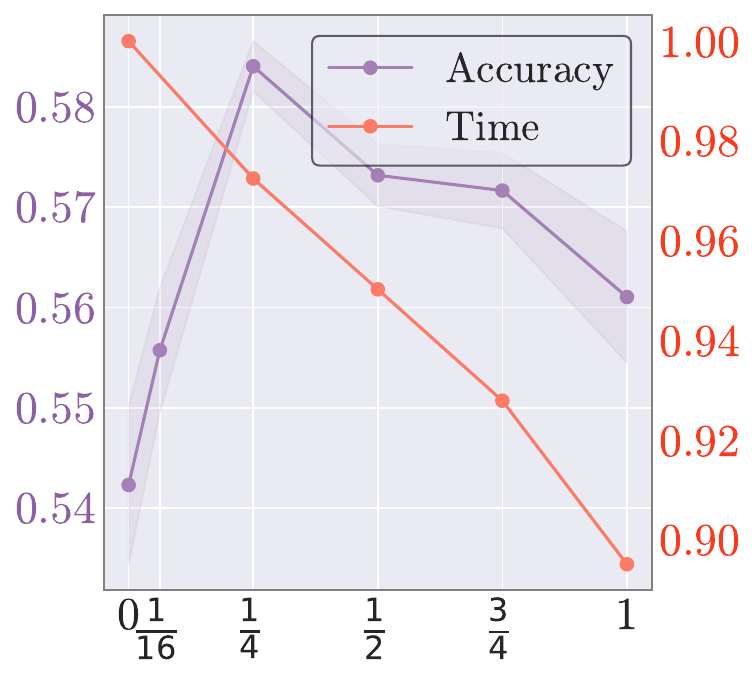}
        \label{fig:ablation_subfig1}
    }\hfill
    \subfigure{
        \includegraphics[width=0.465\linewidth]{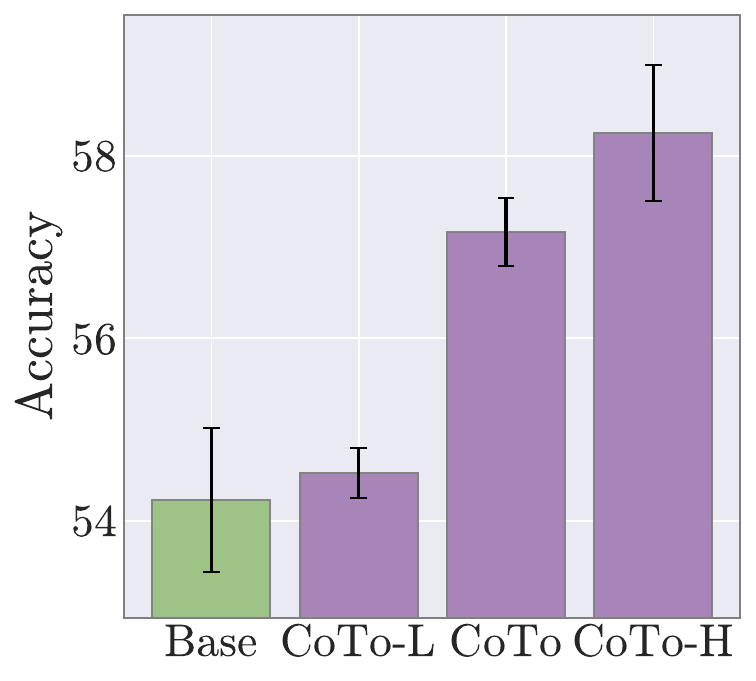}
        \label{fig:ablation_subfig2}
    }
    \vskip -0.12in
    \caption{Ablation analysis of CoTo. Left panel: Impact of varying the proportion of training time allocated to the first (stochastic activation) phase on model accuracy (purple) and normalized training time (orange). Right panel: Comparison of dropout strategies, including no dropout ($\mathrm{Base}$)), nested dropout from lower layers (CoTo-{L}), uniform dropout (CoTo), and nested dropout from higher layers (CoTo-{H}).}
    \vskip -0.1in
    \label{fig:mainfig}
\end{figure}

\noindent{\bf Dropout Strategy.} 
CoTo applies uniform dropout probability across all layers. To assess layer-specific effect, we design two variants: 1) CoTo-{L}, where adapters in lower layers are deactivated first, \ie, at time $t$, only adapters in layers $i>i_0$  (for some threshold $i_0$ determined by $p(t)$) remain active, and 2) CoTo-{H}, where adapters in higher layers are deactivated first, so that early in training only lower‐layer adapters participate. The results reveal that CoTo and CoTo-{H} outperform CoTo-{L}, indicating that randomly deactivating---or prioritizing deactivation of---high‐layer adapters is more effective. Nest‐dropping from lower layers forces the model to rely prematurely on higher‐layer adapters and undermines balanced optimization.

\begin{table}[t]
\caption{Average accuracy (\%) on mathematical reasoning tasks~\citep{cobbe2021training} for LoRA-Pro across different adapter ranks, learning rates, and insertion modules.}
\centering
\vskip 0.08in
\resizebox{\linewidth}{!}{
\begin{tabular}{lccc}
\toprule
LoRA Rank & 8 & 32  & 128 \\ 
\midrule
\ \ LoRA-Pro  & 54.23 \small{± 0.79} & 55.14 \small {± 1.73} & 56.48 \small{± 0.23} \\
\ \  + CoTo  & 57.16 \small{± 0.38} & 57.24 \small{± 0.06} & \textbf{58.50} \small{± 0.46}\\
\toprule
Learning Rate & 5e-5 & 1e-4  & 2e-4 \\ 
\midrule
\ \  LoRA-Pro  &  55.70 \small{± 0.96} & 55.85 \small{± 0.74} & 40.91 \small{± 1.09} \\
\ \  + CoTo  & 55.83 \small{± 0.38} & \textbf{57.16} \small{± 0.38} & 56.25 \small{± 0.53}\\
\toprule
Insertion Module & Attention  & Projection & Gating \\ 
\midrule
\ \  LoRA-Pro  & 45.44 \small{± 0.40} & 49.08 \small{± 0.64} &  48.40 \small{± 0.16}\\
\ \  + CoTo  & 52.41 \small{± 0.56} &  \textbf{54.06} \small{± 0.80}&  51.96 \small{± 0.28} \\
\bottomrule
\end{tabular}
}
\label{tab:llm_ablation}
\end{table}

\noindent{\bf Hyperparameter Sensitivity.} 
To demonstrate CoTo's compatibility with diverse LoRA configurations, we fine‐tune LLaMA-2-7B using three adapter ranks ($r = 8, 32, 128$), three learning rates (5e-5, 1e-4, 2e-4), and three choices of insertion modules from attention, projection, and gating layers, respectively. From Table~\ref{tab:llm_ablation}, we find that LoRA-Pro-CoTo consistently outperforms LoRA-Pro in all settings. 

Further, Figure~\ref{fig:merging-pruning-results} extends this analysis to merging and pruning on the visual texture classification task~\citep{cimpoi2014describing} across five learning rates. CoTo‐trained adapters yield higher merging accuracy and maintain stronger robustness under structured pruning. These trends underscore CoTo's generality: it benefits LoRA variants across a wide spectrum of hyperparameter settings.

\begin{figure}
    \centering
    \includegraphics[width=\linewidth]{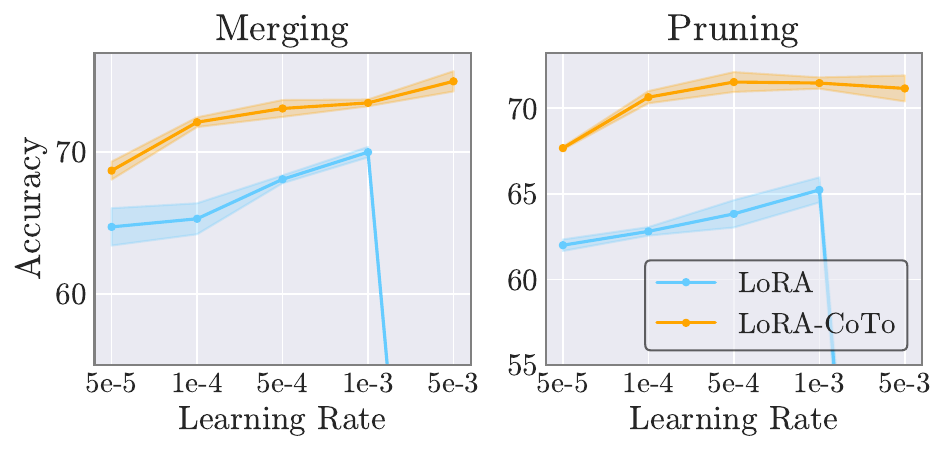}
    \vskip -0.1in
    \caption{Average accuracy (\%) on model merging and pruning for the visual texture classification task~\citep{cimpoi2014describing}. Left panel: Merging accuracy when $\lambda = 0.5$. Right panel: Pruning accuracy when removing alternating layers (\ie, $\mathrm{EveryOther}$).}
    \label{fig:merging-pruning-results}
    \vskip -0.13in
\end{figure}

\noindent\textbf{Training Overhead Reduction.} 
Because CoTo stochastically deactivates adapters in early iterations, it reduces both forward and backward computation. We compare end‐to‐end fine‐tuning times for LoRA, DoRA~\citep{liudora}, and HiRA~\citep{anonymous2024hira}, all under identical hardware and batch‐size settings. From Table~\ref{tab:llm_time}, we observe noticeable training overhead reduction, which arises because, when an adapter is inactive (\ie, $\delta_i = 0$), its low‐rank matrices are skipped entirely. Variants with larger adapter footprints (\eg, DoRA and HiRA) thus experience more pronounced computational savings. Importantly, these gains accrue early in training yet do not compromise---or even improve---final accuracy.

\begin{table}[t]
\centering
\setlength{\tabcolsep}{10.5pt}
\caption{Wall-clock training times for LoRA, DoRA, and HiRA on mathematical reasoning tasks~\citep{cobbe2021training} using a single NVIDIA A6000 GPU.}
\vskip 0.08in
\resizebox{\linewidth}{!}{
\begin{tabular}{lccc}
\toprule
 & LoRA & DoRA  & HiRA \\ 
\midrule
w/o CoTo  & 7h 38min & 19h 00min  & 11h 39min\\
w/ CoTo  & 7h 05min & 14h 30min  & 8h 50min\\
{Speedup}  & {7.20\%} & {23.69\%}  & {24.21\%} \\
\bottomrule
\end{tabular}
}
\vskip -0.08in
\label{tab:llm_time}
\end{table}
\section{Conclusion and Discussion}
We have introduced CoTo, a progressive training strategy for LoRA. By gradually increasing adapter activation probability during training, CoTo promotes more balanced optimization across all layers while encouraging broader exploration of the loss landscape. Our theoretical analysis showed that CoTo enhances layer-wise dropout stability and LMC, while the cooperative-game perspective provided quantitative insights into each adapter's marginal contribution. Extensive experiments across vision-language models, large language models, and diffusion models consistently validated the effectiveness of CoTo. 

While CoTo integrates seamlessly with diverse LoRA variants, a promising direction is to identify the optimal combination of existing LoRA ``tricks.'' For instance, one could explore jointly applying CoTo's progressive schedule with adaptive-rank schemes (like AdaLoRA~\citep{zhang2023adaptive} or ALoRA~\citep{liu2024alora}), weight-decomposed updates (in DoRA~\citep{liudora}), or Hadamard-based high-rank adaptations (in HiRA~\citep{anonymous2024hira}), to determine how these techniques interact and where synergies arise. Systematically evaluating such combinations---potentially via automated hyperparameter search over activation schedules, rank allocations, and initialization strategies---could reveal configurations that maximize performance while minimizing parameter count and compute. Moreover, extending CoTo to jointly optimize over multiple PEFT objectives (\eg, balancing dropout stability, quantization compatibility, and rank efficiency) could yield a unified framework that adapts to various resource constraints and task requirements.

\section*{Acknowledgement}

This work was supported in part by the National Key Research and Development Program of China (No. 2022ZD0160300) and the National Natural Science Foundation of China (No. 62136005).

\section*{Impact Statement}
This work aims to advance the field of machine learning. While it may carry various societal implications, none of which we feel must be highlighted here.


\bibliographystyle{icml2025}

\newpage
\appendix
\onecolumn

\section{Training Details}\label{sec:appendix-training-details}
\subsection{Implementation Details}
CoTo is implemented as a lightweight \texttt{TrainerCallback} within the PEFT ecosystem\footnote{\hyperlink{https://github.com/huggingface/peft}{https://github.com/huggingface/peft}}, enabling seamless integration with any LoRA-style fine-tuning loop. Algorithm~\ref{alg:coto} summarizes its computational procedure.

\begin{algorithm}[t]
    \caption{CoTo}
    \label{alg:coto}
    \begin{algorithmic}[1]
        \REQUIRE Foundation model parameters $\boldsymbol{\theta}_0$, LoRA adapters $\Delta \boldsymbol{\theta} = \{\Delta \mathbf{W}_i\}_{i=1}^{L}$, and the total number of training steps $T$
        \FOR{$t = 1, \ldots, T$}
            \STATE Compute activation probability: $p = \min\left\{1, \frac{4t}{3T}\right\}$
            \STATE Draw a vector $\boldsymbol{\eta} = [\eta_1, \ldots,\eta_L]^\intercal$, where each $\eta_i$ is sampled independently from the uniform distribution $\mathcal{U}(0,1)$ 
            \FOR{each layer $i = 1, \dots, L$}
                \STATE Set adapter state in layer $i$: $\delta_i = \mathbb{I}[\eta_i \leq p]$ 
            \ENDFOR
            \STATE Compute the prediction $\hat{\boldsymbol{y}} = f\left(\boldsymbol{x}_0; \left\{\mathbf{W}_i + \delta_i\mathbf{1} \odot\Delta \mathbf{W}_i\right\}_{i=1}^{L}\right)$
            \STATE Compute the loss $\ell(\hat{\boldsymbol{y}}, \boldsymbol{y})$
            \STATE Compute the gradient $\nabla_{\{\Delta \mathbf{W}_i\}}\ell(\hat{\boldsymbol{y}}, \boldsymbol{y})$ and update the model parameters $\{\Delta \mathbf{W}_i\}$
        \ENDFOR
    \end{algorithmic}
\end{algorithm}

\subsection{Datasets}
We evaluate CoTo across five computational prediction tasks: 1) image classification, 2) commonsense reasoning, 3) mathematical reasoning, 4) language understanding, and 5) image generation. For \textit{image classification}, we follow~\citet{zanella2024low} and use $11$ datasets:
\begin{itemize}
    \item Aircraft~\citep{maji2013fine} (aircraft classification)
    \item Caltech~\citep{fei2004learning} (object recognition)
    \item Cars~\citep{krause20133d} (car classification)
    \item DTD~\citep{cimpoi2014describing} (visual texture classification)
    \item EuroSAT~\citep{helber2019eurosat} (satellite land classification)
    \item Flowers~\citep{nilsback2008automated} (flower classification)
    \item Food~\citep{bossard2014food} (food classification)
    \item ImageNet~\citep{deng2009imagenet} (large-scale object recognition)
    \item Pets~\citep{parkhi2012cats} (pet breed classification)
    \item SUN~\citep{xiao2010sun} (scene recognition)
    \item UCF~\citep{soomro2012ucf101} (human action classification)
\end{itemize}
For \textit{commonsense reasoning}, we use $8$ tasks from {Commonsense170K}~\citep{hu2023llm}:
\begin{itemize}
    \item ARC-c and ARC-e~\citep{arc-ce} (science questions)
    \item BoolQ~\citep{boolq} (yes/no questions)
    \item HellaSwag~\citep{hella} (commonsense inference)
    \item OBQA~\citep{obqa} (multi-step reasoning)
    \item PIQA~\citep{piqa} (physical commonsense reasoning)
    \item SIQA~\citep{siqa} (social reasoning)
    \item WinoGrande~\citep{winoGrande} (fill-in-the-blank questions)
\end{itemize}
For \textit{mathematical reasoning}, we fine-tune on {MetaMathQA}~\citep{yumetamath} and test on  {GSM8K}~\citep{cobbe2021training}. For \textit{language understanding}, we follow~\citet{zhao2024merging} and use $9$ tasks from {GLUE}~\citep{wang2019glue} and Flan Collection~\citep{longpre2023flan}:
\begin{itemize}
    \item CoLA~\citep{dolan2005automatically} (linguistic acceptability)
    \item MNLI~\citep{williams2018broad} (multi-genre natural language inference)
    \item MRPC~\citep{dolan2005automatically} (paraphrase detection)
    \item QNLI~\citep{rajpurkar2016squad} (question-answering)
    \item QQP\footnote{\hyperlink{https://data.quora.com/First-Quora-Dataset-Release-Question-Pairs}{https://data.quora.com/First-Quora-Dataset-Release-Question-Pairs}} (Quora questions)
    \item RTE~\citep{dagan2006pascal, bar2006second, giampiccolo2007third, bentivogli2009fifth} (textual entailment recognition)
    \item SNLI~\citep{bowman2015large} (natural language inference)
    \item SST2~\citep{socher2013recursive} (sentiment analysis)
    \item WNLI~\citep{levesque2011winograd} (coreference resolution)
\end{itemize}
For \textit{image generation}, we train on two content categories from DreamBooth~\citep{ruiz2023dreambooth} and two artistic styles from Hugging Face’s LoRA the Explorer~\footnote{\hyperlink{https://huggingface.co/spaces/multimodalart/LoraTheExplorer}{https://huggingface.co/spaces/multimodalart/LoraTheExplorer}}.
    
\begin{table*}[t]
\caption{Grid-searched learning rates for $11$ image classification tasks~\citep{zanella2024low} across different LoRA variants.}
\label{tab:lr4vlm}
\vskip 0.08in
\centering
\setlength{\tabcolsep}{8.5pt}
\resizebox{\linewidth}{!}{
\begin{tabular}{lccccccccccc}
\toprule
Method & Aircraft & Caltech & Cars & DTD & EuroSAT & Flowers & Food & ImageNet & Pets & SUN & UCF \\
\midrule
LoRA        & 2e-4 & 2e-4 & 2e-4 & 2e-4 & 2e-4 & 2e-4 & 2e-4 & 2e-4 & 2e-4 & 2e-4 & 2e-4 \\
LoRA-CoTo   & 5e-4 & 5e-4 & 5e-4 & 5e-4 & 5e-4 & 5e-4 & 5e-4 & 5e-4 & 5e-4 & 5e-4 & 5e-4 \\
DoRA        & 1e-3 & 1e-4 & 1e-3 & 1e-4 & 1e-4 & 1e-3 & 1e-4 & 1e-4 & 1e-4 & 1e-4 & 1e-4 \\
DoRA-CoTo   & 1e-3 & 2e-4 & 1e-3 & 2e-4 & 1e-3 & 1e-3 & 2e-4 & 2e-4 & 2e-4 & 2e-4 & 2e-4 \\
HiRA        & 5e-3 & 1e-3 & 5e-3 & 1e-3 & 5e-3 & 5e-3 & 1e-3 & 5e-3 & 1e-3 & 1e-3 & 1e-3 \\
HiRA-CoTo   & 1e-2 & 5e-3 & 5e-2 & 5e-3 & 1e-2 & 5e-2 & 5e-3 & 5e-3 & 5e-3 & 5e-3 & 5e-3 \\
\bottomrule
\end{tabular}
}
\end{table*}

\begin{table*}[t]
\caption{Standard deviation of classification accuracies across three random seeds for $11$ image classification tasks.}
\centering
\vskip 0.08in
\setlength{\tabcolsep}{8pt}
\resizebox{\linewidth}{!}{
\begin{tabular}{lccccccccccc}
\toprule
Method & Aircraft & Caltech & Cars & DTD & EuroSAT & Flowers & Food & ImageNet & Pets & SUN & UCF \\
\midrule
LoRA        & 0.75 & 0.12 & 0.24 & 0.97 & 0.81 & 0.19 & 0.22 & 0.06 & 0.27 & 0.18 & 0.44 \\
LoRA-CoTo   & 0.72 & 0.10 & 0.29 & 0.39 & 0.44 & 0.36 & 0.18 & 0.04 & 0.39 & 0.20 & 0.22 \\
DoRA        & 1.04 & 0.32 & 0.32 & 1.39 & 1.09 & 0.18 & 0.20 & 0.10 & 0.28 & 0.10 & 0.22 \\
DoRA-CoTo   & 0.31 & 0.20 & 0.19 & 0.62 & 0.41 & 0.12 & 0.11 & 0.09 & 0.11 & 0.21 & 0.17 \\
HiRA        & 0.27 & 0.09 & 0.25 & 1.33 & 1.33 & 0.02 & 0.09 & 0.14 & 0.09 & 0.11 & 0.23 \\
HiRA-CoTo   & 0.15 & 0.11 & 0.08 & 0.28 & 0.05 & 0.15 & 0.12 & 0.10 & 0.36 & 0.13 & 0.35 \\
\bottomrule
\end{tabular}
}
\label{tab:std_image_classification}
\end{table*}

\subsection{Hyperparameters}
We use publicly available implementations of CLIP-LoRA~\citep{zanella2024low}, DoRA~\citep{liudora}, HiRA~\citep{anonymous2024hira}, LoRA-Pro~\citep{wang2024lorapro}, LoRA-LEGO~\citep{zhao2024merging}, and ZipLoRA~\citep{shah2025ziplora}, retaining original hyperparameters unless otherwise specified.
Table~\ref{tab:lora_config} details the adapter rank, learning rate, insertion module, and dropout configurations for each task. Notably, HiRA requires a significantly higher learning rate ($10$–$20\times$ the default) for convergence. For \textit{image classification}, we grid search learning rates around each method’s default (see Table~\ref{tab:lr4vlm}), while other tasks employ one initial learning rate paired with a cosine annealing schedule.

\begin{table}[t]
\centering
\caption{LoRA hyperparameter configurations for different tasks. The dropout rate is predefined in the LoRAConfig class and applied independently of the proposed CoTo.}
\vskip 0.08in
\label{tab:lora_config}
\setlength{\tabcolsep}{9pt}
\resizebox{\linewidth}{!}{
\begin{tabular}{lccccc}
\toprule
Task & Rank & Default Learning Rate & CoTo Learning Rate & Insertion Module & Dropout Rate \\
\midrule
Image Classification & 2 & 5e-5 & 2e-4 &  Attention  & 0 \\
Commonsense Reasoning & 32 & 1e-5 & 5e-5 & Attention, Projection & 0.05\\
Mathematical Reasoning & 8 & 2e-5 & 1e-4 & Attention, Projection, Gating & 0.1\\
Language Understanding & 8 & / & 2e-5 & Attention ($\mathbf{Q}, \mathbf{V}$) & 0.05 \\
Image Generation & 16 & 1e-5 & 5e-5 & Attention  & 0 \\
\bottomrule
\end{tabular}
}
\vskip -0.08in
\end{table}

\section{Additional Experimental Results}

\subsection{Image Classification}
Standard deviations across $3$ seeds remain low ($<1.4\%$) for all image classification tasks (see Table~\ref{tab:std_image_classification}), confirming result reliability in Table~\ref{tab:vlm_few_shot_results}.

\subsection{Single-Task Merging for Commonsense Reasoning}
Per-task interpolation curves (see Figure~\ref{fig:commonsense-lmc2}) show that CoTo consistently outperforms vanilla LoRA across all $8$ commonsense reasoning tasks, with flatter loss basins.

\begin{table*}[t]
    \centering
    \caption{Average accuracy (\%) on multi-task merging for $6$ discriminative language understanding tasks~\citep{wang2019glue} using the DeBERTa-v3~\citep{hedebertav3} backbone.
    }
    \vskip 0.05in
    \setlength{\tabcolsep}{12pt}
    \resizebox{\linewidth}{!}{
    \ \begin{tabular}{lllccccccc}
    \toprule
     & Learning Rate & Method  & CoLA & MRPC & QNLI & QQP & RTE & SST2 & Avg\\
    \midrule
    \multirow{8}{*}{\rotatebox{90}{w/o CoTo}} 
    & \multirow{4}{*}{5e-4} 
        & LoRA        & 87.44 & 89.46 & 94.31 & 91.05 & 85.56 & 95.18 & 90.50 \\
    & & Fusion       & 69.89 & 68.38 & 49.97 & 65.50 & 47.29 & 55.05 & 59.35 \\
    & & Ensemble     & 69.13 & 31.62 & 50.54 & 63.31 & 52.71 & 50.92 & 53.04 \\
    & & LoRA-LEGO    & 73.28 & 33.15 & 74.19 & 80.95 & 61.46 & 69.18 & 65.37 \\
    \cmidrule(lr){2-10}
    
    & \multirow{4}{*}{1e-3} 
        & LoRA        & 86.48 & 88.48 & 93.87 & 91.16 & 84.12 & 94.95 & 89.84 \\
    & & Fusion       & 69.13 & 68.38 & 49.46 & 63.18 & 47.29 & 50.92 & 58.06 \\
    & & Ensemble     & 69.13 & 68.38 & 49.46 & 63.18 & 47.29 & 50.92 & 58.06 \\
    & & LoRA-LEGO    & 73.54 & 69.32 & 79.54 & 78.67 & 51.99 & 54.13 & 67.86 \\
    
    \midrule

    \multirow{8}{*}{\rotatebox{90}{w/ CoTo}} 
    & \multirow{4}{*}{5e-4} 
        & LoRA        & 86.48 & 89.22 & 93.94 & 90.12 & 82.67 & 94.95 & 89.56 \\
    & & Fusion       & 69.22 & 68.38 & 67.12 & 70.32 & 47.29 & 58.14 & 63.41 \\
    & & Ensemble     & 70.66 & 40.69 & 51.31 & 66.71 & 47.29 & 66.40 & 57.18 \\
    & & LoRA-LEGO    & 53.75 & 70.77 & 76.19 & 75.26 & 64.31 & 89.01 & 71.55 \\
    \cmidrule(lr){2-10}
    
    & \multirow{4}{*}{1e-3} 
        & LoRA   & 86.58 & 89.95 & 93.92 & 90.49 & 83.03 & 95.18 & 89.86 \\
    & & Linear Fusion& 73.15 & 68.38 & 51.91 & 69.96 & 49.10 & 63.30 & 62.63 \\
    & & Ensemble     & 72.39 & 68.38 & 51.42 & 66.98 & 48.38 & 61.58 & 61.52 \\
    & & LoRA-LEGO    & 71.84 & 72.39 & 73.15 & 72.95 & 63.51 & 87.24 & 73.51 \\
    \bottomrule
\end{tabular}
}
\label{tab:llmloradeberta}
\end{table*}

\begin{figure}[t]
    \centering
    \includegraphics[width=\linewidth]{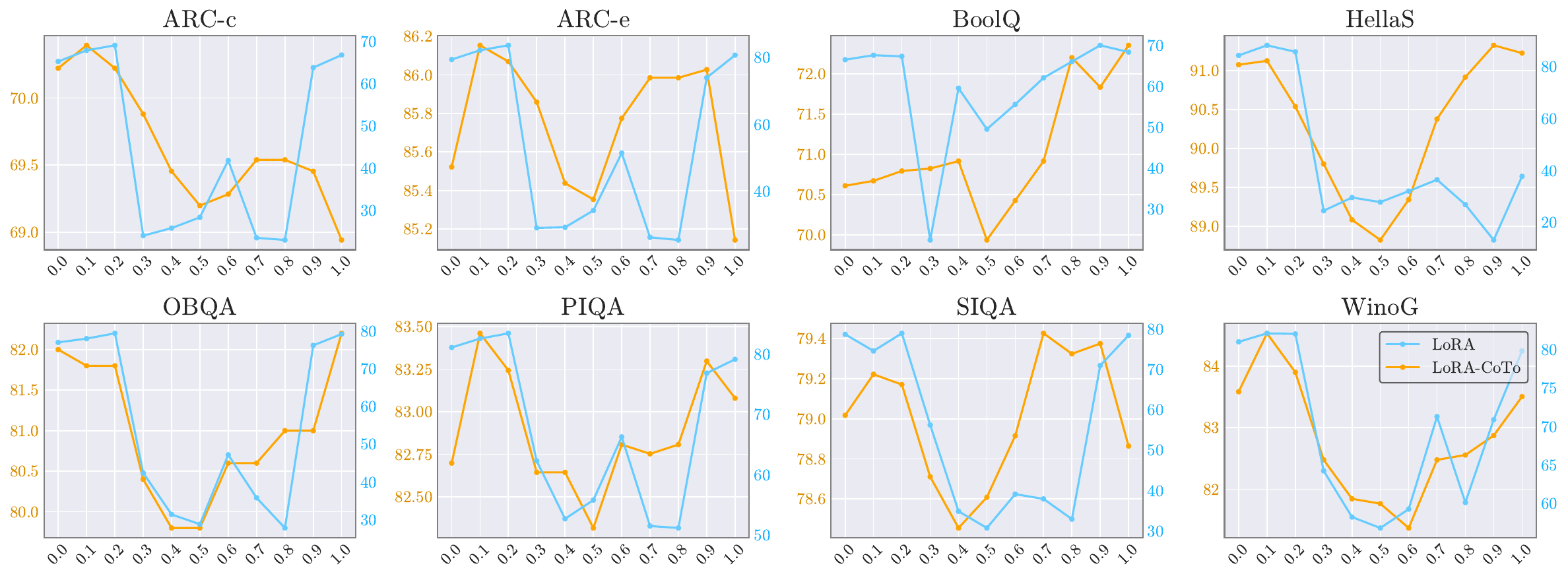}
    \vskip -0.1in
    \caption{ Linear interpolation accuracy on $8$ individual commonsense reasoning tasks~\citep{hu2023llm}.}
    \vskip -0.1in
    \label{fig:commonsense-lmc2}
\end{figure}

\subsection{Multi-Task Merging for Discriminative Language Understanding}
CoTo improves merging accuracy by $4.18\%$–$6.18\%$ across linear weight fusion, linear model ensemble, and LoRA-LEGO~\citep{zhao2024merging} strategies (see Table~\ref{tab:llmloradeberta}), validating benefits in discriminative language tasks using the DeBERTa-v3~\citep{hedebertav3} backbone.

\begin{figure*}[t]
    \centering
    \includegraphics[width=\linewidth]{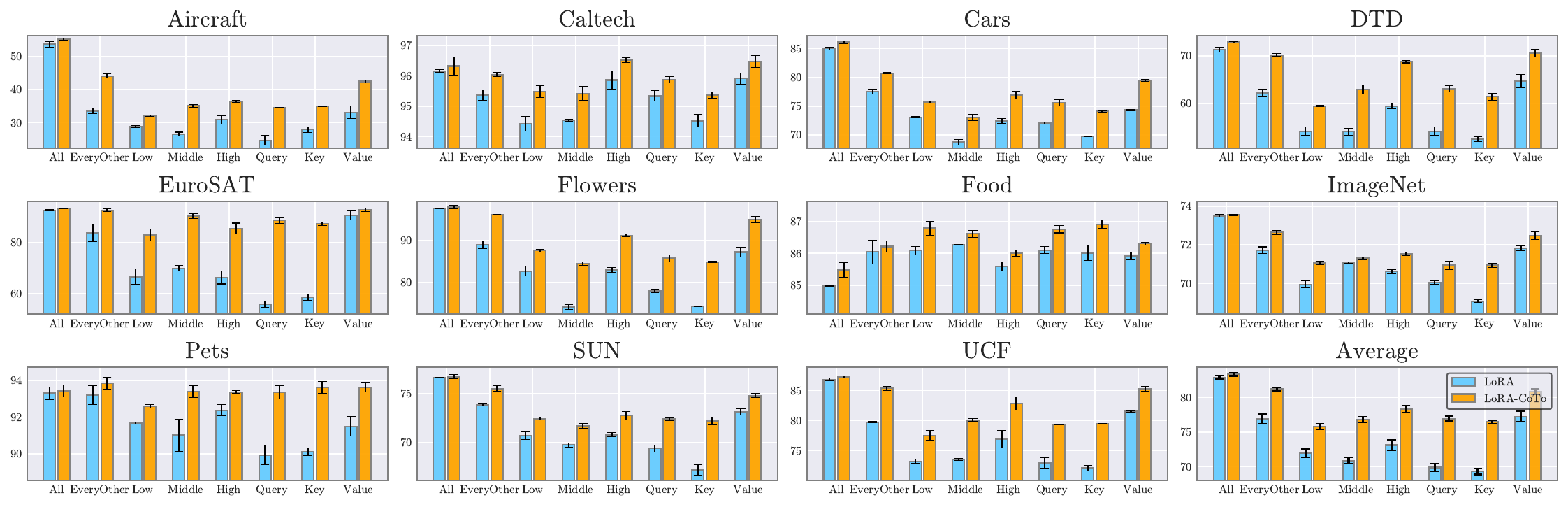}
    \vskip -0.05in
    \caption{ Average accuracy (\%) on structured model pruning for $11$ image classification tasks.}
    \label{fig:pruning-complete-dataset}
\end{figure*}

\begin{figure*}[t]
    \centering
    \includegraphics[width=\linewidth]{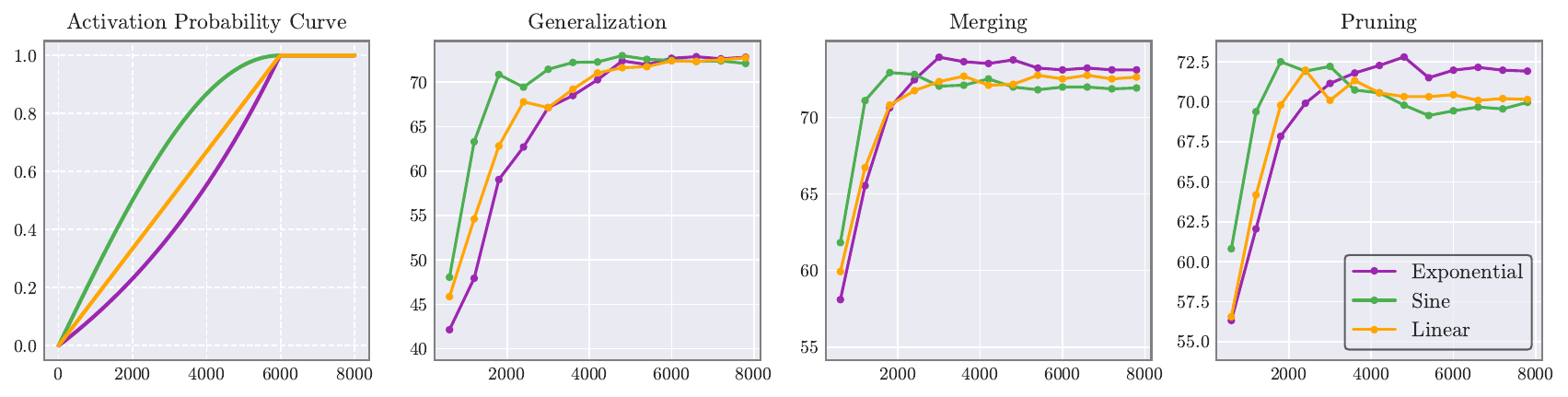}
    \vskip -0.05in
    \caption{Performance evolution during training under different activation schedules on the visual texture classification task~\citep{cimpoi2014describing}. Merging accuracy is measured at $\lambda =0.5$. Pruning accuracy is measured for $\mathrm{EveryOther}$ (\ie, removing alternating layers).}
    \label{fig:activation-results}
\end{figure*}

\subsection{Structured Pruning for Image Classification}
CoTo maintains higher accuracy than vanilla LoRA under all structured pruning strategies across $11$ image classification tasks (Figure~\ref{fig:pruning-complete-dataset}).

\section{Extended Analysis}
\subsection{Activation Schedule Ablation on DTD}\label{sec:activation-prob}
Linear activation (CoTo default) balances convergence speed and robustness (see Figure~\ref{fig:activation-results}). Exponential schedules improve merging/pruning but delay early convergence; sine schedules accelerate convergence but reduce robustness.

\subsection{Weight Convergence Visualization}
\begin{figure*}[t]
    \centering
    \includegraphics[width=\linewidth]{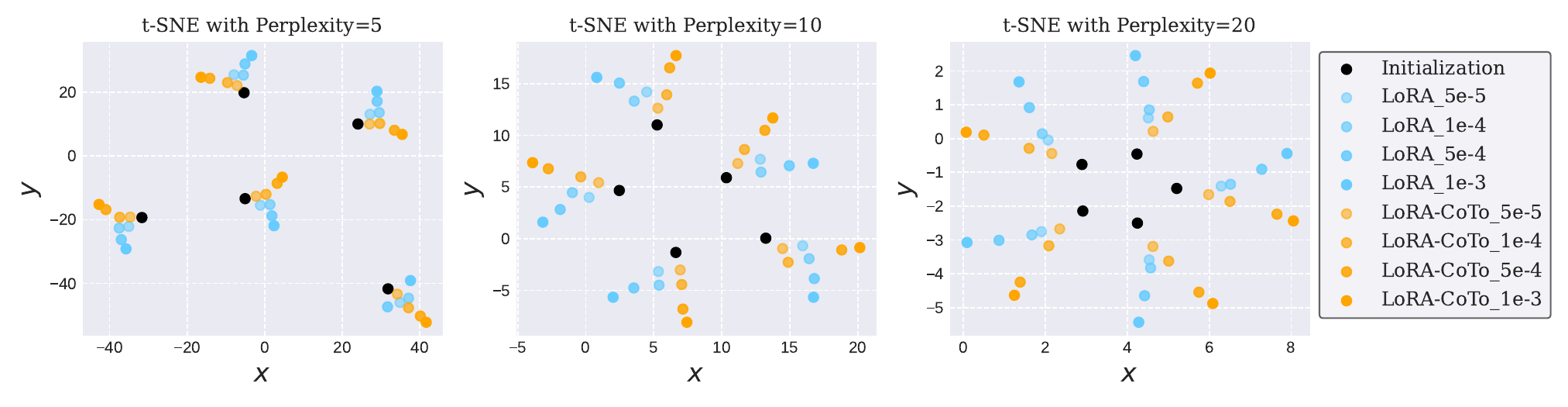}
    \vskip -0.2in
    \caption{t-SNE visualization~\citep{van2008visualizing} of the learned weights of LoRA adapters across five seeds and four learning rates, under three perplexity settings. Black dots denote initialization points, and color gradients indicate different learning rates.}
    \label{fig:tsne}
\end{figure*}
t-SNE plots (see Figure~\ref{fig:tsne}) confirm LoRA adapters converge near initialization (\ie, ``lazy training'') across learning rates. CoTo yields tighter clusters under initialization noise.

\section{Proof of Theorem~\ref{mytheoremone}}\label{sec:appendix-theory}
\mytheoremone*

\begin{proof}
\begin{equation}
\begin{aligned}
\mathbb{E}_{\boldsymbol{\delta}}\left[\ell\left(\hat{\boldsymbol{y}}, \boldsymbol{y}\right)\right] & =\mathbb{E}_{\boldsymbol{\delta}}\left[\ell\left(f\left(\boldsymbol{x}_0; \left\{\mathbf{W}_i + \delta_i\mathbf{1} \odot\Delta \mathbf{W}_i\right\}_{i=1}^{L}\right), \boldsymbol{y}\right)\right] \\
& =\sum_j\binom{L}{j} p^j(1-p)^{L-j} \mathbb{E}_{\|\boldsymbol{\delta}\|_1=j}\left[\ell\left(f\left(\boldsymbol{x}_0; \left\{\mathbf{W}_i + \delta_i \mathbf{1}\odot\Delta \mathbf{W}_i\right\}_{i=1}^{L}\right), \boldsymbol{y}\right)\right] \\
& \geq \sum_j w_j(p) \ell\left(\mathbb{E}_{\|\boldsymbol{\delta}\|_1=j} \left[f\left(\boldsymbol{x}_0; \left\{\mathbf{W}_i + \delta_i \mathbf{1}\odot\Delta \mathbf{W}_i\right\}_{i=1}^{L}\right)\right], \boldsymbol{y}\right)\\
& =\sum_j w_j(p) \ell\left(\tilde{\boldsymbol{y}}_j, \boldsymbol{y}\right).
\end{aligned}    
\end{equation}
\end{proof}

\end{document}